\documentclass[twoside,11pt]{article}

\usepackage{jmlr2e}
\usepackage{hyperref}
\usepackage{amsmath}
\usepackage{subfigure}
\usepackage{epstopdf}
\usepackage{url}
\usepackage{multirow}

\begin{document}

\title{Understanding Boltzmann Machine and Deep Learning via A Confident Information First Principle}

\author{\name Xiaozhao Zhao \email 0.25eye@gmail.com \\
       \addr School of Computer Science and Technology, Tianjin University \\Tianjin, 300072, China
       \AND
       \name Yuexian Hou \email krete1941@gmail.com \\
       \addr School of Computer Science and Technology, Tianjin University\\Tianjin, 300072, China
       \AND
       \name Qian Yu \email yqcloud@gmail.com \\
       \addr School of Computer Software, Tianjin University\\Tianjin, 300072, China
       \AND
       \name Dawei Song \email dawei.song2010@gmail.com \\
       \addr School of Computer Science and Technology, Tianjin University\\Tianjin, 300072, China \\
       and Department of Computing, The Open University\\Milton Keynes, MK7 6AA, UK
       \AND
       \name Wenjie Li \email cswjli@comp.polyu.edu.hk \\
       \addr Department of Computing, The Hong Kong Polytechnic University\\
       Hung Hom, Kowloon, Hong Kong, China
       }

\editor{}

\maketitle

\begin{abstract}
Typical dimensionality reduction methods focus on directly reducing the number of random variables while retaining maximal variations in the data. In this paper, we consider the dimensionality reduction in parameter spaces of binary multivariate distributions. We propose a general Confident-Information-First (CIF) principle to maximally preserve parameters with confident estimates and rule out unreliable or noisy parameters. Formally, the confidence of a parameter can be assessed by its Fisher information, which establishes a connection with the inverse variance of any unbiased estimate for the parameter via the Cram\'{e}r-Rao bound. We then revisit Boltzmann machines (BM) and theoretically show that both single-layer BM without hidden units (SBM) and restricted BM (RBM) can be solidly derived using the CIF principle. This can not only help us uncover and formalize the essential parts of the target density that SBM and RBM capture, but also suggest that the deep neural network consisting of several layers of RBM can be seen as the layer-wise application of CIF. Guided by the theoretical analysis, we develop a sample-specific CIF-based contrastive divergence (CD-CIF) algorithm for SBM and a CIF-based iterative projection procedure (IP) for RBM. Both CD-CIF and IP are studied in a series of density estimation experiments.
\end{abstract}


\section{Introduction}\label{sec:intro}
Recently, deep learning models (e.g., Deep Belief Networks (DBN)\citep{hinton06}, Stacked Denoising Auto-encoder \citep{ranzato07autoencoder}, Deep Boltzmann Machine (DBM) \citep{Salakhutdinov2012} and etc.) have drawn increasing attention due to their impressive results in various application areas, such as computer vision \citep{Bengio06greedy,ranzato07autoencoder,osindero07image}, natural language processing \citep{Collobert08nlp} and information retrieval \citep{Salakhutdinov07kernel, Salakhutdinov07sigir}.
Despite of these practical successes, there have been debates on the fundamental principle of the design and training of those deep architectures.
One important problem is the \emph{unsupervised pre-training}, which would fit the parameters to better capture the structure in the input distribution \citep{bengio2010why}. From the probabilistic modeling perspective, this process can be interpreted as an attempt to recover a set of model parameters for a generative neural network that would well describe the distribution underlying the observed high-dimensional data.
In general, to sufficiently depict the original high-dimensional data requires a high-dimensional parameter space.
However, overfitting usually occur when the model is excessively complex.
On the other hand, \citet{Bengio13big} empirically shows the failure of some big neural networks in leveraging the added capacity to reduce underfitting.

Hence it is important to uncover and understand \emph{what the first principle would be on reducing the dimensionality of the parameter space concerning the sample size}.
This question is also recognized by \citet{bengio2010why}. They empirically show that the unsupervised pre-training acts as a regularization on parameters in a way that the parameters are set in a region, from which better basins of attraction can be reached. The regularization on parameters could be seen as a kind of dimensionality reduction procedure on parameter spaces by restricting those parameters in a desired region. However, the intrinsic mechanisms behind the regularization process are still unclear. Thus, further theoretical justifications are needed in order to formally analyze what the essential parts of the target density that the neural networks can capture. An indepth investigation into this question will lead to two significant results: 1) a formal explanation on \emph{what} exactly the neural networks would perform in the pre-training process; 2) some theoretical insights on \emph{how} to do it better.

\citet{samplecomplexity} empirically show that the sampling density of a given dataset and the resulting complexity of a learning problem are closely interrelated. If the initial sampling density is insufficient, this may result in a preferred model of a
lower complexity, so that we have a satisfactory sampling to estimate model parameters. On the other hand, if the number of samples is originally abundant, the preferred model may become more complex so that we could have represented the dataset in more details.
Moreover, this connection becomes more complicated if the observed samples contain noises. Now the obstacle is how to incorporate this relationship between the dataset and the preferred model into the learning process.
In this paper, we mainly focus on analyzing the Boltzmann machines, the main building blocks of many neural networks, from a novel information geometry (IG)\citep{Amari93} perspective.

Assuming there exists an ideal parametric model $S$ that is general enough to represent all system phenomena, our goal of the parametric reduction is to derive a lower-dimensional sub-model $M$ for a given dataset (usually insufficient or perturbed by noises) by reducing the number of free parameters in $S$. In this paper, we propose a \emph{Confident-Information-First} (CIF) principle to maximally preserve the parameters with highly confident estimates and rule out the unreliable or noisy parameters with respect to the density estimation of binary multivariate distributions. From the IG point of view, the \emph{confidence}\footnote{Note that, in this paper, the meaning of \emph{confidence} is different from the common concept \emph{degree of confidence} in statistics.} of a parameter can be assessed by its \emph{Fisher information} \citep{Amari93}, which establishes a connection with the inverse variance of any unbiased estimate for the considered parameter via the Cram\'{e}r-Rao bound \citep[see][]{rao45attainable}.
It is worth emphasizing that the proposed CIF as a principle of parametric reduction is fundamentally different from the traditional feature reduction (or feature extraction) methods \citep{Fodor02asurvey,lee07nonlinear}. The latter focus on directly reducing the dimensionality on feature space by retaining maximal variations in the data, e.g., Principle Components Analysis (PCA) \citep{abdi10pcareview}, while CIF offers a principled method to
deal with high-dimensional data in the parameter spaces by a strategy that is universally derived from the first principle, independent of the geometric metric in the feature spaces.

The CIF takes an information-oriented viewpoint of statistical machine learning. The \emph{information}\footnote{There are many kinds of ``information'' defined for a probability distribution $p(x)$, e.g., entropy, Fisher information. The entropy is a global measure of smoothness in $p(x)$. Fisher information is a local measure that is sensitive to local rearrangement of points $(x,p(x))$.} is rooted on the variations (or fluctuations) in the \emph{imperfect} observations (due to insufficient sampling, noise and intrinsic limitations of the ``observer'') and transmits throughout the whole learning process. This idea is also well recognized in modern physics, as stated in \citet{wheeler94}:
``All things physical are information-theoretic in origin and this is a participatory universe...Observer participancy gives rise to information; and information gives rise to physics.''.
Following this viewpoint, \citet{scienceoffisherinformation} unifies the derivation of physical laws in major fields of physics, from the Dirac equation to the Maxwell-Boltzmann velocity dispersion law, using the
extreme physical information principle (EPI). The information used in this unification is exactly the Fisher information \citep{rao45attainable}, which measures the quality of any measurement(s).

In terms of statistical machine learning, the aim of this paper is of two folds: a) to incorporate the Fisher information into the modelling of the intrinsic variations in the data that give rise to the desired model, by using the IG framework \citep{Amari93}; b) to show by examples that some existing probabilistic models, e.g., SBM and RBM, comply with the CIF principle and can be derived from it.
The main contributions are:
\vspace{-1mm}
\begin{enumerate}
  \item We propose a \emph{general CIF} principle for parameter reduction to maximally preserve the parameters of high confidence and eliminate the unreliable parameters of binary multivariate distributions in the framework of IG. We also give a geometric interpretation of CIF by showing that it can maximally preserve the expected information distance.
  \item The implementation of CIF, that is, the derivation of probabilistic models, is illustrated by revisiting two widely used Boltzmann machines, i.e., Single layer BM without hidden units (SBM) and restricted BM (RBM). And the deep neural network consisting of several layers of RBM can be seen as the layer-by-layer application of CIF.
  \item Based on the above theoretical analysis, a CIF-based iterative projection procedure (IP) inherent in the learning of RBM is uncovered. And traditional gradient-based methods, e.g., maximum-likelihood (ML) and contrastive divergence (CD) \citep{hinton05CD}, can be seen as approximations of IP. Experimental results indicate that IP is more robust against sampling biases, due to its separation of the positive sampling process and the gradient estimation.
  \item Beyond the general CIF, we propose a \emph{sample-specific} CIF and integrate it into the CD algorithm for SBM to confine the parameter space to a confident parameter region indicated by samples. It leads to a significant improvement in a series of density estimation experiments, when the sampling is insufficient.
\end{enumerate}

The rest of the paper is organized as follows: Section \ref{sec:theoryIG} introduces some preliminaries of IG. Then the general CIF principle is proposed in Section \ref{sec:cif}. In Section \ref{sec:implementationCIF}, we analyze two implementations of CIF using the BM with and without hidden units, i.e., SBM and RBM.
After that, a sample-specific CIF-based CD learning method (CD-CIF) for SBM and a CIF-based iterative projection procedure for RBM are proposed and experimentally studied in Section \ref{sec:cd-cif}.
Finally, we draw conclusions and highlight some future research directions in Section \ref{sec:conclusions}.

\section{Theoretical Foundations of Information Geometry}\label{sec:theoryIG}
In this section, we introduce and develop the theoretical foundations of Information Geometry (IG) \citep{Amari93} for the manifold $S$ of binary multivariate distributions with a given number of variables $n$, i.e., the open simplex of all probability distributions over binary vector $x \in \{0,1\}^{n}$. This will lay the foundation for our theoretical deviation of the \emph{general CIF}.
\subsection{Notations for Manifold S}\label{sec:def}
In IG, a family of probability distributions is considered as a differentiable manifold with certain parametric coordinate systems. In the case of binary multivariate distributions, four basic coordinate systems are often used: $p$-coordinates, $\eta$-coordinates, $\theta$-coordinates and mixed-coordinates \citep{Amari93,hou2013}. Mixed-coordinates is of vital importance for our analysis.

For the $p$-coordinates $[p]$ with $n$ binary variables, the probability distribution over $2^n$ states of $x$ can be completely specified by any $2^n-1$ positive numbers indicating the probability of the corresponding exclusive states on $n$ binary variables. For example, the $p$-coordinates of $n=2$ variables could be $[p]=(p_{01}, p_{10}, p_{11})$. Note that IG requires all probability terms to be positive.

For simplicity, we use the capital letters $I,J,\dots$ to index the coordinate parameters of probabilistic distribution. An index $I$ can be regarded as a subset of $\{1,2,\dots,n\}$. And $p_I$ stands for the probability that all variables indicated by $I$ equal to one and the complemented variables are zero. For example, if $I=\{1,2,4\}$ and $n=4$, then $p_I=p_{1101}=Prob(x_1=1,x_2=1,x_3=0,x_4=1)$. Note that the null set can also be a legal index of the $p$-coordinates, which indicates the probability that all variables are zero, denoted as $p_{0 \dots 0}$.

Another coordinate system often used in IG is $\eta$-coordinates, which is defined by:
\begin{equation}\label{eq:etacoordinate}
    \eta_I=E[X_I]=Prob\{\prod_{i\in I}x_i = 1\}
\end{equation}
where the value of $X_I$ is given by $\prod_{i\in I}x_i$ and the expectation is taken with respect to the probability distribution over $x$. Grouping the coordinates by their orders, $\eta$-coordinate system is denoted as $[\eta]=(\eta^1_i, \eta^2_{ij},\dots, \eta^n_{1,2...n})$, where the superscript indicates the order number of the corresponding parameter. For example, $\eta^2_{ij}$ denotes the set of all $\eta$ parameters with the order number $2$.

The $\theta$-coordinates (natural coordinates) is defined by:
\begin{equation}\label{eq:thetacoordinate}
    \log{p(x)}=\sum_{I\subseteq\{1,2,\dots,n\}, I\neq NullSet}{\theta^I X_I} - \psi
\end{equation}
where $\psi=-\log Prob\{x_i=0, \forall i\in \{1,2,...,n\}\}$. The $\theta$-coordinate is denoted as $[\theta]=(\theta^{i}_1, \theta^{ij}_2,\dots, \theta^{1,...,n}_n)$, where the subscript indicates the order number of the corresponding parameter.
Note that the order indices locate at different positions in $[\eta]$ and $[\theta]$ following the convention in \citet{amari92igbm}.

The relation between coordinate systems $[\eta]$ and $[\theta]$ is bijective \citep{amari92igbm}.
More formally, they are connected by the Legendre transformation:
\begin{equation}\label{eq:trans_eta_theta}
    \theta^I=\frac{\partial \phi(\eta)}{\partial \eta_I}, \eta_I=\frac{\partial \psi(\theta)}{\partial \theta^I}
\end{equation}
where $\psi(\theta)$ and $\phi(\eta)$ meet the following identity
\begin{equation}\label{eq:legedre}
    \psi(\theta)+\phi(\eta)-\sum{\theta^I\eta_I}=0
\end{equation}
The function $\psi(\theta)$ is introduced in Eq. (\ref{eq:thetacoordinate}):
\begin{equation}\label{eq:psi_theta}
    \psi(\theta)=\log(\sum_x{exp\{\sum_I{\theta^I X_I(x)}\}})
\end{equation}
and hence $\phi(\eta)$ is the negative of entropy:
\begin{equation}\label{eq:phi_eta}
    \phi(\eta)=\sum_x{p(x;\theta(\eta))\log{p(x;\theta(\eta))}}
\end{equation}

Next we introduce mixed-coordinates, which is important for our derivation of CIF. In general, the manifold $S$ of probability distributions could be represented by the $l$-mixed-coordinates \citep{amari92igbm}:
\begin{equation}\label{eq:lmixedcoordinates}
  [\zeta]_l=(\eta^1_i, \eta^2_{ij},\dots, \eta^l_{i,j,...,k}, \theta^{i,j,...,k}_{l+1},\dots, \theta^{1,...,n}_n)
\end{equation}
where the first part consists of $\eta$-coordinates with order less or equal to $l$ (denoted by $[\eta^{l-}]$) and the second part consists of $\theta$-coordinates with order greater than $l$ (denoted by $[\theta_{l+}]$), $l \in \{1,...,n\}$.

\subsection{Fisher Information Matrix for Parametric Coordinates} \label{sec:fisherinformationmatrix}
For a general coordinate system $[\xi]$, the $i$th-row and $j$th-col element of the Fisher information matrix for $[\xi]$ (denoted by $G_{\xi}$) is defined as the covariance of the scores of $[\xi_i]$ and $[\xi_j]$ \citep{rao45attainable}, i.e., $$g_{ij}=E[\frac{\partial \log p(x;\xi)}{\partial \xi_i}\cdot \frac{\partial \log p(x;\xi)}{\partial \xi_j}]$$ The Fisher information measures the amount of information in the data that a statistic carries about the unknown parameter \citep{kass89asymptotic}.
The Fisher information matrix is of vital importance to our analysis because the inverse of Fisher information matrix gives an asymptotically tight lower bound of the covariance matrix of any unbiased estimate for the considered parameters \citep{rao45attainable}. 
Moreover, the Fisher information matrix, as a distance metric, is invariant to re-parameterization \citep{rao45attainable}, and can be proved to be the unique metric that is invariant to the map of random variables corresponding to a sufficient statistic \citep{amari92igbm,chentsov80}.

Another important concept related to our analysis is the \emph{orthogonality} defined by Fisher information. Two coordinate parameters $\xi_i$ and $\xi_j$ are called \emph{orthogonal} if and only if their Fisher information vanishes, i.e., $g_{ij}=0$, meaning that their influences to the log likelihood function are uncorrelated. A more technical meaning of orthogonality is that the \emph{maximum likelihood estimates} (MLE) of orthogonal parameters can be independently performed.

The Fisher information for $[\theta]$ can be rewritten as $g_{IJ}=\frac{\partial^2 \psi(\theta)}{\partial \theta ^I \partial \theta ^J}$, and for $[\eta]$ it is $g^{IJ}=\frac{\partial^2 \phi(\eta)}{\partial \eta_I \partial \eta_J}$
\citep{Amari93}. Let $G_\theta=(g_{IJ})$ and $G_\eta =(g^{IJ})$ be the Fisher information matrices for $[\theta]$ and $[\eta]$ respectively. It can be shown that $G_\theta$ and $G_\eta$ are mutually inverse matrices, i.e., $\sum_J{g^{IJ}g_{JK}}=\delta^I_K$, where $\delta^I_K=1$ if $I=K$ and 0 otherwise \citep{Amari93}.

In order to generally compute $G_\theta$ and $G_\eta$, we develop the following Proposition \ref{prop:fishermatrix} and \ref{prop:fishermatrix_eta}. Note that Proposition \ref{prop:fishermatrix} is a generalization of Theorem 2 in \citet{amari92igbm}.
\begin{proposition}\label{prop:fishermatrix}
The Fisher information between two parameters $\theta^I$ and $\theta^J$ in $[\theta]$, is given by
\begin{equation}\label{eq:proposiationFisherMetrictheta}
    g_{IJ}(\theta)=\eta_{I\bigcup J}-\eta_I \eta_J
\end{equation}
\end{proposition}
\begin{proof}
in Appendix \ref{appendix:thetafisher}
\end{proof}

\vspace{-3mm}
\begin{proposition}\label{prop:fishermatrix_eta}

The Fisher information between two parameters $\eta_I$ and $\eta_J$ in $[\eta]$, is given by
\begin{equation}\label{eq:proposiationFisherMetriceta}
    g^{IJ}(\eta)=\sum_{K\subseteq I\cap J}{(-1)^{|I-K|+|J-K|} \cdot \frac{1}{p_{K}}}
\end{equation}
where $|\cdot|$ denotes the cardinality operator.
\end{proposition}
\begin{proof}
in Appendix \ref{appendix:etafisher}
\end{proof}
For example, let $[p]=(p_{001}, p_{010}, p_{011}, p_{100}, p_{101}, p_{110}, p_{111})$ be the $p$-coordinates of a probability distribution with three variables. Then, the Fisher information of $\eta_I$ and $\eta_J$  can be calculated based on Equation (\ref{eq:proposiationFisherMetriceta}): if $I=\{1,2\}$ and $J=\{2,3\}$, $g^{IJ}=\frac{1}{p_{000}}+\frac{1}{p_{010}}$, and if $I=\{1,2\}$ and $J=\{1,2,3\}$, $g^{IJ}=-(\frac{1}{p_{000}}+\frac{1}{p_{010}}+\frac{1}{p_{100}}+\frac{1}{p_{110}})$, etc.

\section{The General CIF Principle For Parametric Reduction}\label{sec:cif}
As described in Section \ref{sec:def}, the general manifold $S$ of all probability distributions over binary vector $x \in \{0,1\}^{n}$, could be exactly represented using the parametric coordinate systems of dimensionality $2^{n} -1$. However, given the limited samples generated from a target distribution, it is almost infeasible to determine its coordinates in such high-dimensional parameter space with acceptable accuracy and in reasonable time. Given a target distribution $q(x)$ on the general manifold $S$, we consider the problem of realizing it by a lower-dimensionality submanifold. This is defined as the problem of parametric reduction for multivariate binary distributions.

\subsection{The General CIF Principle}
In this section, we will formally illuminate the general CIF for parametric reduction. Intuitively, if we can construct a coordinate system so that the confidences (measured by Fisher information) of its parameters entail a natural hierarchy, in which high confident parameters can be significantly distinguished from low confident ones, then the general CIF can be conveniently implemented by keeping the high confident parameters unchanging and setting the lowly confident parameters to neutral values. However, the choice of coordinates (or equivalently, parameters) in CIF is crucial to its usage. This strategy is infeasible in terms of $p$-coordinates, $\eta$-coordinates or $\theta$-coordinates, since it is easy to see that the hierarchies of confidences in these coordinate systems are far from significant, as shown by an example later in this section.

The following propositions show that mixed-coordinates meet the requirement realizing the general CIF. Let $[\zeta]_l$ be the mixed-coordinates defined in Section \ref{sec:def}. Proposition \ref{prop:fishermatrix_mix} gives a closed form for calculating the Fisher information matrix $G_\zeta$.

\begin{proposition}\label{prop:fishermatrix_mix}
The Fisher information matrix of the $l$-mixed-coordinates $[\zeta]_l$ is given by:
    \begin{equation}\label{eq:estimationerror}
        G_{\zeta}= \left(
                     \begin{array}{cc}
                       A & 0 \\
                        0 & B \\
                     \end{array}
                   \right)
    \end{equation}
where $A=((G_\eta^{-1})_{I_\eta})^{-1}$, $B=((G_\theta^{-1})_{J_\theta})^{-1}$, $G_\eta$ and $G_\theta$ are the Fisher information matrices of $[\theta]$ and $[\zeta]_l$, respectively, and $I_\eta$ is the index of the parameters shared by $[\eta]$ and $[\zeta]_l$, i.e., $\{\eta^1_i,..., \eta^l_{ij}\}$, and $J_\theta$ is the index set of the parameters shared by $[\theta]$ and $[\zeta]_l$, i.e., $\{\theta^{i,j,...,k}_{l+1},\dots, \theta^{1,...,n}_n\}$.
\end{proposition}
\vspace{-3mm}
\begin{proof}
in Appendix \ref{appendix:mixfisher}.
\end{proof}
\vspace{-3mm}
\begin{proposition}\label{prop:fishermatrix_mixDiagonal}
The diagonal of $A$ are lower bounded by $1$, and that of $B$ are upper bounded by $1$.
\end{proposition}
\vspace{-3mm}
\begin{proof}
in Appendix \ref{appendix:etafisherDiagonal}.
\end{proof}

According to Proposition \ref{prop:fishermatrix_mix} and Proposition \ref{prop:fishermatrix_mixDiagonal}, the confidences of coordinate parameters in $[\zeta]_l$ entail a natural hierarchy: the first part of high confident parameters $[\eta^{l-}]$ are separated from the second part of low confident parameters $[\theta_{l+}]$, which has a neutral value (zero).
Moreover, the parameters in $[\eta^{l-}]$ is orthogonal to the ones in $[\theta_{l+}]$, indicating that we could estimate these two parts independently \citep{hou2013}. Hence we can implement the general CIF principle for parametric reduction in $[\zeta]_l$ by replacing low confident parameters with neutral value zeros and reconstructing the resulting distribution. It turns out that the submanifold tailored by CIF becomes $[\zeta]_{l_t}=(\eta_{i}^1,..., \eta_{ij...k}^l,0,\dots,0)$. We call $[\zeta]_{l_t}$ the $l$-tailored-mixed-coordinates.

To grasp an intuitive picture for the general CIF strategy and its significance w.r.t mixed-coordinates, let us consider an example with $[p]=( p_{001}=0.15, p_{010}=0.1, p_{011}=0.05, p_{100}=0.2, p_{101}=0.1, p_{110}=0.05, p_{111}=0.3)$. Then the confidences for coordinates in $[\eta]$, $[\theta]$ and $[\zeta]_2$ are given by the diagonal elements of the corresponding Fisher information matrices. Applying the $2$-tailored CIF in mixed-coordinates, the loss ratios of Fisher information is $0.001\%$, and the ratio of the Fisher information of the tailored parameter ($\theta_3^{123}$) to the remaining $\eta$ parameter with the smallest Fisher information is $0.06\%$. On the other hand, the above two ratios become $7.58\%$ and $94.45\%$ (in $\eta$-coordinates) or $12.94\%$ and $92.31\%$ (in $\theta$-coordinates), respectively.

\begin{figure}
  \centering
  \includegraphics[width=0.4\textwidth]{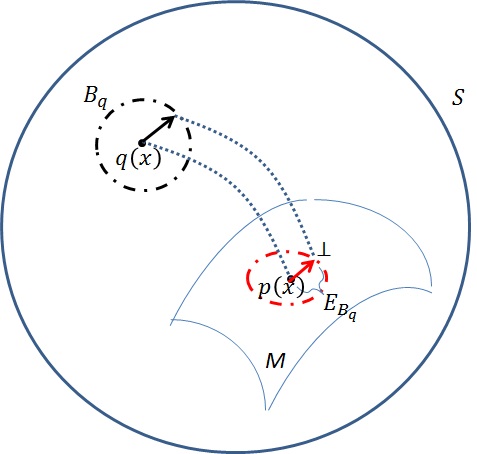}
  \caption{By projecting a point $q(x)$ on $S$ to a submanifold $M$, the $l$-tailored-mixed-coordinates $[\zeta]_{l_t}$ gives a desirable $M$ that maximally preserve the expected Fisher information distance when projecting a $\varepsilon$-neighborhood centered at $q(x)$ onto $M$.}
  \label{fig:expectedFD} 
\end{figure}

Next, we will restate the CIF principle in terms of the geometric perspective of submanifold projection.
Let $M$ be a smooth submanifold in $S$. Given a point $q(x)\in S$, the projection of $q(x)$ to $M$ is the point $p(x)$ that belongs to $M$ and is closest to $q(x)$ in the sense of the Kullback-Leibler divergence (K-L divergence) from the distribution $q(x)$ to $p(x)$ \citep{amari92igbm}:
\begin{equation}\label{eq:fisherinformationdivergence}
  D(q(x),p(x))=\sum_{x} q(x) \log \frac{q(x)}{p(x)}
\end{equation}
Alternatively, since K-L divergence is not symmetric, the projection of $q(x)$ to $M$ can also be defined as the point $p(x)\in M$ that minimizes the K-L divergence from $M$ to $q(x)$. In the rest of this paper, the direction of the K-L divergence used in a particular projection is explicitly specified when there is an ambiguity.

The CIF entails a submanifold of $S$ via the $l$-tailored-mixed-coordinates $[\zeta]_{l_t}$.
However, there exist many different submanifolds of $S$. Now our question is: \emph{does there exist a general criterion to distinguish which projection is best? If such principle does exist, is CIF the right one?} The following proposition shows that the general CIF entails a geometric interpretation illuminated in Figure \ref{fig:expectedFD}, which would lead to an optimal submanifold $M$.\footnote{Note that the CIF is related to but fundamentally different from the $m$-projection in \citet{amari92igbm}. \citet{amari92igbm} focuses on the problem of projecting a point $Q$ on $S$ to the submanifold of BM and shows that $m$-projection is the point on BM that is closest to $Q$. Actually, the $m$-projection is a special case of our $[\zeta]_{l_t}$-projection when $l$ is 2. In the present paper, we focus on the problem of developing a general criterion that could help us find the optimal submanifold to project on.}
\begin{proposition}\label{prop:GeometricView}
Given a statistical manifold $S$ in $l$-mixed-coordinates $[\zeta]_l$, let the corresponding $l$-tailored-mixed-coordinates $[\zeta]_{l_t}$ has $k$ free parameters. Then, among all $k$-dimensional submanifolds of $S$, the submanifold determined by $[\zeta]_{l_t}$ can maximally preserve the expected information distance
induced by Fisher-Rao metric.
\end{proposition}
\begin{proof}
in Appendix \ref{appendix:geometriccif}.
\end{proof}
\section{Two Implementations of CIF using Boltzmann Machine}\label{sec:implementationCIF}
In previous section, a general CIF is uncovered in the $[\zeta]_l$ coordinates for multivariate binary distributions. Now we consider the implementations of CIF when $l$ equals to 2 using the Boltzmann machines (BM).
More specifically, we show that two kinds of BMs, i.e., the single layer BM without hidden units (SBM) and the restricted BM (RBM), are indeed instances following the general CIF principle. For each case, the application of CIF can be interpreted in two perspectives: an algebraic and geometric interpretation.

\subsection{Neural Networks as Parametric Reduction Model}\label{subsec:notationBM}
Many neural networks with fixed architecture, such as SBM, RBM, high-order BM \citep{highorderBM}, deep belief networks \citep{hinton06}, have been proposed to approximately realize the underlying distributions in different application scenarios. Those neural networks are designed to fulfill the parametric reduction for certain tasks by specifying the number of adjustable parameters, namely the number of connection weights and the number of biases. We believe that there exists a general criterion to design the structure of neural submanifolds for the particular application in hand, and the problem of parametric reduction is equivalent to the choice of submanifolds.
Next, we will briefly introduce the general BM and the gradient-based learning algorithm.

\subsubsection{Introduction To The Boltzmann Machines}
In general, a BM \citep{ackley85BM} is defined as a stochastic neural network consisting of visible units $x\in\{0,1\}^{n_x}$ and hidden units $h\in\{0,1\}^{n_h}$, where each unit fires stochastically depending on the weighted sum of its inputs. The energy function is defined as follows:
\begin{equation}\label{eq:energyBM}
   E_{BM}(x,h;\xi)=-\frac{1}{2}x^TUx-\frac{1}{2}h^TVh -x^TWh -b^Tx-d^Th
\end{equation}
where $\xi=\{U,V,W,b,d\}$ are the parameters: visible-visible interactions ($U$), hidden-hidden interactions ($V$), visible-hidden interactions ($W$), visible self-connections ($b$) and hidden self-connections ($d$). The diagonals of $U$ and $V$ are set to zero.
We can express the Boltzmann distribution over the joint space of $x$ and $h$ as below:
\begin{equation}\label{eq:probBM}
    p(x,h;\xi)=\frac{1}{Z}exp\{-E_{BM}(x,h;\xi)\}
\end{equation}
where $Z$ is a normalization factor.

Let $B$ be the set of Boltzmann distributions realized by BM. Actually, $B$ is a submanifold of the general manifold $S_{xh}$ over $\{x,h\}$. From Equation (\ref{eq:probBM}) and (\ref{eq:energyBM}), we can see that $\xi=\{U,V,W,b,d\}$ plays the role of $B$'s coordinates in $\theta$-coordinates (Equation \ref{eq:thetacoordinate}) as follows:
\begin{eqnarray}\label{eq:bmtotheta}
  \theta_1 &:& \theta^{x_i}_1 = b_{x_i}, \theta^{h_j}_1 = d_{h_j} (\forall x_i\in x, h_j \in h) \nonumber \\
  \theta_2 &:& \theta^{x_ix_j}_2 = U_{x_i,x_j}, \theta^{x_ih_j}_2 = W_{x_i,h_j},\theta^{h_ih_j}_2 = V_{h_i,h_j} (\forall x_i, x_j\in x; h_i,h_j \in h) \nonumber \\
  \theta_{2+} &:& \theta^{x_i\dots x_j h_u \dots h_v}_m = 0, m>2, (\forall x_i,\dots, x_j \in x; h_u,\dots,h_v \in h)
\end{eqnarray}
So the $\theta$-coordinates for BM is given by:
\begin{equation}\label{eq:thetaforBM}
  [\theta]_{BM}=(\underbrace{\theta^{x_i}_1, \theta^{h_j}_1}_{1-order}, \underbrace{\theta^{x_ix_j}_2, \theta^{x_ih_j}_2, \theta^{h_ih_j}_2}_{2-order}, \underbrace{0, \dots, 0 }_{orders>2}).
\end{equation}

The SBM and RBM are special cases of the general BM. Since SBM has $n_h=0$ and all the visible units are connected to each other, the parameters of SBM are $\xi_{sbm}=\{U,b\}$ and $\{V,W,d\}$ are all set to zero. For RBM, it has connections only between hidden and visible units. Thus, the parameters of RBM are $\xi_{rbm}=\{W,b,d\}$ and $\{U,V\}$ are set to zero.

\subsubsection{Formulation on the Gradient-based Learning of BM} \label{sec:learningMLforBM}
Given the sample $\underline{x}$ that generated from the underlying distribution, the \emph{maximum-likelihood} (ML) is commonly used gradient ascent method for training BM in order to maximize the log-likelihood $\log p(\underline{x};\xi)$ of the parameters $\xi$ \citep{hinton05CD}. Based on Equation (\ref{eq:probBM}), the log-likelihood is given as follows:
$$\log p(\underline{x};\xi)=log \sum_h e^{-E(\underline{x},h;\xi)} - log \sum_{x',h'} e^{-E(x',h';\xi)}$$
Differentiating the log-likelihood, the gradient vector with respect to $\xi$ is as follows:
\begin{equation}\label{eq:gradientloglikelihood}
  \frac{\partial \log p(\underline{x};\xi)}{\partial \xi}=\sum_h p(h|\underline{x};\xi)\frac{\partial [-E(\underline{x},h;\xi)]}{\partial \xi} - \sum_{x',h'} p(h'|x';\xi)\frac{\partial [-E(x',h';\xi)]}{\partial \xi}
\end{equation}
The $\frac{\partial E(x,h;\xi)}{\partial \xi}$ can be easily calculated from Equation (\ref{eq:energyBM}). Then we can obtain the stochastic gradient using Gibbs sampling \citep{gilk96mcmc} in two phases: sample $\underline{h}$ given $\underline{x}$ for the first term, called the positive phase, and sample $(\underline{x}',\underline{h}')$ from the stationary distribution $p(x',h';\xi)$ for the second term, called the negative phase. Now with the resulting stochastic gradient estimation, the learning rule is to adjust $\xi$ by:
\begin{equation}\label{eq:learnrulestochasticgradient}
  \Delta \xi \! =\! \varepsilon \cdot \frac{\partial \log p(\underline{x};\xi)}{\partial \xi}=\varepsilon \cdot (- \langle \frac{\partial E(\underline{x},\underline{h};\xi)}{\partial \xi} \rangle_0 + \langle \frac{\partial E(\underline{x'},\underline{h'};\xi)}{\partial \xi} \rangle_\infty)
\end{equation}
where $\varepsilon$ is the learning rate, $\langle\cdot \rangle_0$ denotes the average using the sample data and $\langle\cdot \rangle_\infty$ denotes the average with respect to the stationary distribution $p(x,h;\xi)$ after the corresponding Gibbs sampling phases.\footnote{Both the two special BMs, i.e., SBM and RBM, can be trained using the ML method. Note that SBM has no hidden units and hence no positive sampling is needed in training SBM.}

In following sections, we will revisit two special BM, namely SBM and RBM, and theoretically show that both SBM and RBM can be derived using the CIF principle. This helps us formalize what essential parts of the target density the SBM and RBM capture.
\subsection{The CIF-based Derivation of Boltzmann Machine without Hidden Units}\label{sec:geometryBM}

Given any underlying probability distribution $q(x)$ on the general manifold $S$ over $\{x\}$, the logarithm of $q(x)$ can be represented by a linear decomposition of $\theta$-coordinates as shown in Equation (\ref{eq:thetacoordinate}).
Since it is impractical to recognize all coordinates for the target distribution, we would like to only approximate part of them and end up with a $k$-dimensional submanifold $M$ of $S$, where $k$ ($\ll2^{n_x}-1$) is the number of free parameters.
Here, we set $k$ to be the same dimensionality as SBM, i.e., $k=\frac{n_x(n_x+1)}{2}$, so that all candidate submanifolds are comparable to the submanifold endowed by SBM (denoted as $M_{sbm}$).
Next, the rationale underlying the design of $M_{sbm}$ can be illuminated using the general CIF from two perspectives, algebraically and geometrically.
\subsubsection{SBM as 2-tailored-mixed-coordinates}
Let the $2$-mixed-coordinates of $q(x)$ on $S$ be $[\zeta]_2=(\eta^1_i, \eta^2_{ij},\theta_3^{i,j,k},\dots, \theta^{1,...,n_x}_{n_x})$.
Applying the general CIF on $[\zeta]_2$, our parametric reduction rule is to preserve the high confident part parameters $[\eta^{2-}]$ and replace low confident parameters $[\theta_{2+}]$ with a fixed neutral value zero. Thus we derive the $2$-tailored-mixed-coordinates:
$[\zeta]_{2_t}=(\eta_{i}^1, \eta_{ij}^2,0,\dots,0)$, as the optimal approximation of $q(x)$ by the $k$-dimensional submanifolds.
On the other hand, given the 2-mixed-coordinates of $q(x)$, the projection $p(x)\in M_{sbm}$ of $q(x)$ is proved to be $[\zeta]_{p}=(\eta_{i}^1, \eta_{ij}^2,0,\dots,0)$ \citep{amari92igbm}. Thus, SBM defines a probabilistic parameter space that is exactly derived from CIF.
\subsubsection{SBM as Maximal Information Distance Projection}
Next corollary, following Proposition \ref{prop:GeometricView}, shows a geometric derivation of SBM.
We make it explicit that the projection on $M_{sbm}$ could maximally preserve the expected information distance comparing to other tailored submanifolds of $S$ with the same dimensionality $k$.

\begin{corollary}\label{prop:SBMCIF}
Given the general manifold $S$ in $2$-mixed-coordinates $[\zeta]_2$, SBM defines an $k$-dimensional submanifold of $S$ that can maximally preserve the expected information distance induced by Fisher-Rao metric.
\end{corollary}
\begin{proof}
in Appendix \ref{appendix:SBMCIF}.
\end{proof}
From the CIF-based derivation, we can see that SBM confines the statistical manifold in the parameter subspace spanned by those directions with high confidences, which is proved to maximally preserve the expected information distance.

\subsubsection{The Relation Between $[\zeta]_{2_t}$ and the ML Learning of SBM}
To learn such $[\zeta]_{2_t}$, we need to learn the parameters $\xi$ of SBM such that its stationary distribution preserves the same coordinates $[\eta^{2-}]$ as target distribution $q(x)$. Actually, this is exactly what traditional gradient-based learning algorithms intend to do to train SBM. Next proposition shows that the ML method for training SBM is equivalent to learn the tailored 2-mixed coordinates $[\zeta]_{2_t}$.

\begin{proposition}\label{prop:sbmmlcloseform}
    Given the target distribution $q(x)$ with 2-mixed coordinates: $$[\zeta]_2=(\eta^1_i, \eta^2_{ij},\theta_{2+}),$$
    the coordinates of the SBM with stationary distribution $q(x;\xi)$, learnt by ML, are uniquely given by $[\zeta]_{2_t}=(\eta_{i}^1, \eta_{ij}^2,\theta_{2+}=0)$
\end{proposition}
\begin{proof}
in Appendix \ref{appendix:sbmmlcloseform}.
\end{proof}

\subsection{The CIF-based Derivation of Restricted Boltzmann Machine}\label{subsec:geometryRBM}
In previous section, the general CIF uncovers why SBM uses the coordinates up to $2^{nd}$-order, i.e., preserves the $\eta$-coordinates  of the $1^{st}$-order and $2^{nd}$-order. In this section, we will investigate the cases where hidden units are introduced.
Particularly, one of the fundamental problem in neural network research is the unsupervised representation learning \citep{Bengio12representationlearning}, which attempts to characterize the underlying distribution through the discovery of a set of latent variables (or features). Many algorithmic models have been proposed, such as restricted Boltzmann machine (RBM) \citep{hinton06} and auto-encoders \citep{Rifai11contractive,Vincent2010Denoising}, for learning one level of feature extraction. Then, in deep learning models, the representation learnt at one level is used as input for learning the next level, etc.
However, some important questions remain to be clarified: Do these algorithms implicitly learn about the whole density or only some aspects? \emph{If they capture the essence of the target density, then can we formalize the link between the essential part and omitted part?} This section will try to answer these questions using CIF.
\subsubsection{Two Necessary Conditions for Representation Learning}\label{sec:cccondition}
In terms of one level feature extraction, there are two main principles that guide a good representation learning:
\begin{itemize}
  \item \emph{Compactness} of representation: minimize the redundancy between hidden variables in the representation \footnote{The concept of compactness in the neural network is of two-folds. 1) model-scale compactness: a restriction on the number of hidden units in order to give a parsimonious representation w.r.t underlying distribution; 2) structural compactness: a restriction on how hidden units are connected such that the redundancy in the hidden representation is minimized. In this paper, we mainly focus on the latter case.}.
  \item \emph{Completeness} of reconstruction: the learnt representation captures sufficient information in the input, and could completely reconstruct input distribution, in a statistical sense.
\end{itemize}

Let $S_{xh}$ be the general manifold of probability distributions over the joint space of visible units $x$ and hidden units $h$, and $S_x$ be the general manifold over visible units $x$.
Given any observation distribution $q(x)\in S_x$, our problem is to \emph{find the $p(x,h)\in S_{xh}$ with the marginal distribution $p(x)$ that best approximates $q(x)$, while $p(x,h)$ is consistent with the compactness and completeness conditions}.
Here, the K-L divergence, defined in Equation (\ref{eq:fisherinformationdivergence}), is used as the criterion of approximation.

First, we will investigate the submanifold of joint distributions $p(x,h)\in S_{xh}$ that fulfill the above two conditions. Let us denote this submanifold as $M_{cc}$. Extending Equation (\ref{eq:thetacoordinate}) to manifold $S_{xh}$, $p(x,h)$ has the $\theta$-coordinates defined by:
\begin{equation}\label{eq:thetacoordinatehidden}
    \log{p(x,h)}=\sum_{I\subseteq\{x,h\}\& I\neq NullSet}{\theta^I X_I} - \psi
\end{equation}
For the \emph{completeness} requirement, it is easy to prove that the probability of any input variable $x_i$ can be fully determined only by the given hidden representation and independent with remainng input variables $x_j(j\neq i)$, if and only if $\theta^I=0$ for any $I$ that contains two or more input variables in Equation (\ref{eq:thetacoordinatehidden}). Similarly, the \emph{compactness} corresponds to the extraction of statistically independent hidden variables given the input, i.e., $\theta^I=0$ for any $I$ that contains two or more hidden variables in Equation (\ref{eq:thetacoordinatehidden}). Then $M_{cc}$ is given by the following coordinate system:
\begin{equation}\label{eq:thetaforCC}
  [\theta]_{cc}=(\underbrace{\theta^{x_i}_1, \theta^{h_j}_1}_{1-order}, \underbrace{\theta^{x_ix_j}_2\!=0, \theta^{x_ih_j}_2, \theta^{h_ih_j}_2\!=0}_{2-order}, \underbrace{0, \dots, 0 }_{orders>2}).
\end{equation}

Then our problem is restated as to \emph{find the $p(x,h)\in M_{cc}$ with the marginal distribution $p(x)$ that best approximates $q(x)$}.

\subsubsection{The Equivalence between RBM and $M_{cc}$} \label{sec:rbmequivalentcc}
RBM is a special kind of BM that restricts the interactions in Equation (\ref{eq:energyBM}) only to those between hidden and visible units, i.e., $U_{x_i,x_j}=0, V_{h_i,h_j}=0~\forall x_i, x_j\in x; h_i,h_j \in h$. Let $\xi_{rbm}=\{W,b,d\}$ denotes the set of parameters in RBM.
Thus, the $\theta$-coordinates for RBM can be derived directly from Equation (\ref{eq:thetaforBM}):
\begin{equation}\label{eq:thetaforRBM}
  [\theta]_{RBM}=(\underbrace{\theta^{x_i}_1, \theta^{h_j}_1}_{1-order}, \underbrace{\theta^{x_ix_j}_2\!=0, \theta^{x_ih_j}_2, \theta^{h_ih_j}_2\!=0}_{2-order}, \underbrace{0, \dots, 0 }_{orders>2})
\end{equation}
Comparing Equation (\ref{eq:thetaforCC}) to (\ref{eq:thetaforRBM}), the submanifold $M_{rbm}$ defined by RBM is equivalent with $M_{cc}$ since they share exactly the same coordinate system.
This indicates that the \emph{compactness} and \emph{completeness} conditions is indeed realized by RBM. We use a simpler notation $B$ to denote $M_{rbm}$.
Next, we will show how to use CIF to interpret the training process of RBM.

\subsubsection{The CIF-based Interpretation on the Learning of RBM} \label{sec:iplearning}
A RBM produces a stationary distribution $p(x,h)\in S_{xh}$ over $\{x,h\}$. However, given the target distribution $q(x)$, only the marginal distribution of RBM over the visible units are specified by $q(x)$, leaving the distributions on hidden units vary freely. Let $H_q$ be the set of probability distributions $q(x,h)\in S_{xh}$ that have the same marginal distribution $q(x)$ and the conditional distributions $q(h_j|x)$ of each hidden unit $h_j$ is realized by the RBM's activation function with parameter $\xi_{rbm}$ (that is the logistic sigmoid activation: $f(h_j|x;\xi_{rbm})=\frac{1}{1+exp\{-\sum_{i\in \{1,\dots,n_x\}} W_{ij}x_i - d_j\}}$):
\begin{equation}\label{eq:submanifoldhq}
  H_q\! =\! \{q(x,h)\in S_{xh}|\exists\xi_{rbm},\!\sum_h{q(x,h)}=q(x),and~q(h|x;\xi_{rbm})\!=\!\!\prod_{h_j \in h}\!\!{f(h_j|x;\xi_{rbm})}\}
\end{equation}
Then our problem in Section \ref{sec:cccondition} is restated with respect to $S_{xh}$: \emph{search for a RBM in $B$ that minimizes the divergence from $H_q$ to $B$} \footnote{This restated problem directly follows from the fact that: the minimum divergence $D(H_q, B)$ in the whole manifold $S_{xh}$ is equal to the minimum divergence $D[q(x),B_{x}]$ in the visible manifold $S_x$, shown in Theorem 7 in \cite{amari92igbm}.}.

Given $p(x,h;\xi_p)$, its best approximation on $H_q$ is defined by the projection $\Gamma_H(p)$, which gives the minimum K-L divergence from $H_q$ to $p(x,h;\xi_p)$. Next proposition shows how the projection $\Gamma_H(p)$ is obtained.

\begin{proposition}\label{prop:hqtorbm}
    Given a distribution $p(x,h;\xi_p)\in B$, the projection $\Gamma_H(p)\in H_q$ that gives the minimum divergence $D(H_q, p(x,h;\xi_p))$ from $H_q$ to $p(x,h;\xi_p)$ is the $q(x,h;\xi_q) \in H_q$ that satisfies $\xi_q=\xi_p$.
\end{proposition}
\begin{proof}
in Appendix \ref{appendix:hptorbm}.
\end{proof}

On the other hand, given $q(x,h;\xi_q)\in H_q$, the best approximation on $B$ is the projection $\Gamma_B(q)$ of $q$ to $B$. In order to obtain an explicit expression of $\Gamma_B(q)$, we introduce the following fractional mixed coordinates $[\zeta^{xh}]$ \footnote{Note that both the fractional mixed coordinates $[\zeta^{xh}]$ and $2$-mixed coordinates $[\zeta]$ are mixtures of $\eta$-coordinates and $\theta$-coordinates. In $[\zeta]$, coordinates of the same order are taken from either $[\eta]$ or $[\theta]$. However, in $[\zeta^{xh}]$, the $2^{nd}$-order coordinates consist of the $\{\eta_{x_ih_j}^2\}$ from $[\eta]$ and $\{\theta^{x_ix_j}_2, \theta^{h_ih_j}_2\}$ from $[\theta]$, that is why the term ``fractional'' is used.} for the general manifold $S_{xh}$:
\begin{equation}\label{eq:mixedcoordinatenew}
  [\zeta^{xh}]=(\underbrace{\eta_{x_i}^1, \eta_{h_j}^1}_{1-order}, \underbrace{\theta^{x_ix_j}_2, \eta_{x_ih_j}^2, \theta^{h_ih_j}_2}_{2-order}, \underbrace{\theta_{2+}}_{orders>2})
\end{equation}
The $[\zeta^{xh}]$ is a valid coordinate system, that is, the relation between $[\theta]$ and $[\zeta^{xh}]$ is bijective. This is shown in the next proposition.
\begin{proposition}\label{prop:fractionalmixbijective}
    The relation between the two coordinate systems $[\theta]$ and $[\zeta^{xh}]$ is bijective.
\end{proposition}
\begin{proof}
in Appendix \ref{appendix:fractionalmixbijective}.
\end{proof}

The next proposition gives an explicit expression of the coordinates for the projection $\Gamma_B(q)$ learnt by RBM using the fractional mixed coordinates in Equation (\ref{eq:mixedcoordinatenew}).
\begin{proposition}\label{prop:projectionRBMcloseform}
    Given $q(x,h;\xi_q) \in H_q$ with fractional mixed coordinates: $$[\zeta^{xh}]_{q}=(\eta_{x_i}^1, \eta_{h_j}^1, \theta^{x_ix_j}_2, \eta_{x_ih_j}^2, \theta^{h_ih_j}_2,\theta_{2+}),$$
    the coordinates of the learnt projection $\Gamma_B(q)$ of $q(x,h;\xi_q)$ on the submanifold $B$ are uniquely given by:
\begin{equation}\label{eq:mixedcoordinatenewProjectionRBM}
  [\zeta^{xh}]_{\Gamma_B(q)}=(\eta_{x_i}^1, \eta_{h_j}^1, \theta^{x_ix_j}_2=0, \eta_{x_ih_j}^2, \theta^{h_ih_j}_2=0, \theta_{2+}=0)
\end{equation}
\end{proposition}
\begin{proof}
This proof comes in three parts:
\begin{enumerate}
  \item the projection $\Gamma_B(q)$ of $q(x,h)$ on $B$ is unique;
  \item this unique projection $\Gamma_B(q)$ can be achieved by minimizing the divergence $D[q(x,h),B]$ using gradient descent method;
  \item The fractional mixed coordinates of $\Gamma_B(q)$ is exactly the one given in Equation (\ref{eq:mixedcoordinatenewProjectionRBM}).
\end{enumerate}
See Appendix \ref{appendix:projectionRBMcloseform} for the detailed proof.
\end{proof}
\begin{figure}
  \centering
  \includegraphics[width=0.4\textwidth]{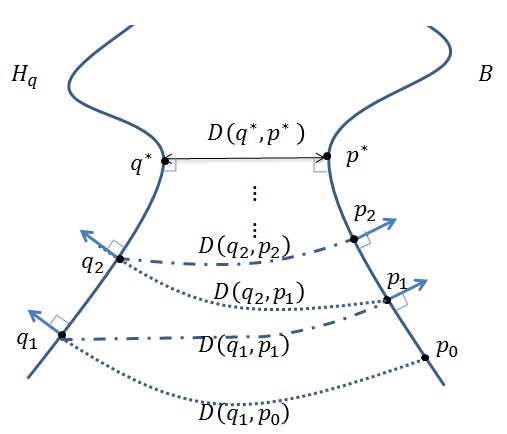}
  \caption{The iterative learning for RBM: in searching for the minimum divergence between $H_q$ and $B$, we first choose an initial RBM $p_0$ and then perform projections $\Gamma_H(p)$ and $\Gamma_B(q)$ iteratively, until the fixed points of the projections $p^*$ and $q^*$ are reached. With different initializations, the iterative projection algorithm may end up with different local minima on $H_q$ and $B$, respectively.}
  \label{fig:iterativelearning} 
\end{figure}

Back to the problem of obtaining the best approximation to the given target $q(x)$, the learning of RBM can be implemented by the following iterative projection process\footnote{\citet{amari92igbm} proposed a similar iterative algorithm framework for the fully-connected BM.
In the present paper, we reformulate this iterative algorithm for the learning of RBM and give explicit expressions of how the projections are achieved.}:

Let $p_0(x,h;\xi_p^0)$ be the initial RBM.
For $i=0,1,2,\dots$,
\begin{enumerate}
  \item Put $q_{i+1}(x,h)=\Gamma_H(p_i(x,h;\xi_p^i))$
  \item Put $p_{i+1}(x,h;\xi_p^{i+1})=\Gamma_B(q_{i+1}(x,h))$
\end{enumerate}
where $\Gamma_H(p)$ denotes the projection of $p(x,h;\xi_p)$ to $H_q$, and $\Gamma_B(q)$ denotes the projection of $q(x,h)$ to $B$.
The iteration ends when we reach the fixed points of the projections $p^*$ and $q^*$, that is $\Gamma_H(p^*)=q^*$ and $\Gamma_B(q^*)=p^*$. The iterative projection process of RBM is illustrated in Figure \ref{fig:iterativelearning}. The convergence property of this iterative algorithm is guaranteed using the following proposition:
\begin{proposition}\label{prop:monotonicdivergence}
    The monotonic relation holds in the iterative learning algorithm:
    \begin{equation}\label{eq:propmonotonic}
      D[q_{i+1},p_i] \geq D[q_{i+1},p_{i+1}] \geq D[q_{i+2},p_{i+1}], \forall i=\{0,1,2,\dots\}
    \end{equation}
    where the equality holds only for the fixed points of the projections.
\end{proposition}
\begin{proof}
in Appendix \ref{appendix:monotonicdivergence}.
\end{proof}

The CIF-based iterative projection procedure (IP) for RBM gives us an alternative way to investigate the learning process of RBM.
The invariance in the learning of RBM is the CIF:
in the $i$th iteration, given $q_i \in H_q$ with fractional mixed coordinates:
    $$[\zeta^{xh}]_{q_i}=(\eta_{x_i}^1, \eta_{h_j}^1, \theta^{x_ix_j}_2, \eta_{x_ih_j}^2, \theta^{h_ih_j}_2,\theta_{2+})$$
then the ordinates of the projection $p_{i}$ on $B$, i.e., $\Gamma_{B}(q_i)$, is given by Equation (\ref{eq:mixedcoordinatenewProjectionRBM}):
    $$[\zeta^{xh}]_{p_i}=(\eta_{x_i}^1, \eta_{h_j}^1, \theta^{x_ix_j}_2=0, \eta_{x_ih_j}^2, \theta^{h_ih_j}_2=0, \theta_{2+}=0)$$
Now we will show that the process of the projection $\Gamma_{B}(q_i)$ can be derived from CIF, i.e., highly confident coordinates $[\eta_{x_i}^1, \eta_{h_j}^1, \eta_{x_ih_j}^2]$ of $q_i$ are preserved while lowly confident coordinates $[\theta_{2+}]$ are set to neutral value zero.
For the fractional mix coordinates system $[\zeta^{xh}]$, the closed form of its Fisher information matrix does not have the good expression formula like Proposition \ref{prop:fishermatrix_mix} which are possessed by the mixed-coordinate system $[\zeta]$.
Next, we will show that the fractional mix-coordinates of $\Gamma_{B}(q_i)$ can be derived in three steps by jointly applying the CIF and the \emph{completeness} and \emph{compactness} conditions.
Let the corresponding 2-mixed $\zeta$-coordinates for $q_i$ be $[\zeta]_{2,q_i}=(\eta_{x_i}^1, \eta_{h_j}^1, \eta_{x_ix_j}^2, \eta_{x_ih_j}^2, \eta_{h_ih_j}^2,\theta_{2+})$. First, we apply the general CIF for parametric reduction in $[\zeta]_{2,q_i}$ by replacing lowly confident coordinates $[\theta_{2+}]$ with neutral value zeros and preserving the remaining coordinates, resulting in the tailored mix-coordinates $[\zeta]_ {2_t,q_i}=(\eta_{x_i}^1, \eta_{h_j}^1, \eta_{x_ix_j}^2, \eta_{x_ih_j}^2, \eta_{h_ih_j}^2,\theta_{2+}=0)$, as described in Section \ref{sec:cif}. Then, we transmit $[\zeta]_ {2_t,q_i}$ into the fractional coordinate system, i.e., $[\zeta^{xh}]_{q_i}=(\eta_{x_i}^1, \eta_{h_j}^1, \theta^{x_ix_j}_2, \eta_{x_ih_j}^2, \theta^{h_ih_j}_2, \theta_{2+}=0)$. Finally, the \emph{completeness} and \emph{compactness} conditions require that the $[\theta^{x_ix_j}_2, \theta^{h_ih_j}_2]$ in $[\zeta^{xh}]_{q_i}$ are also set to neutral value zeros. Hence, we can see that the coordinates of the projection $\Gamma_{B}(q_i)$ is exactly the one given by Equation (\ref{eq:mixedcoordinatenewProjectionRBM}).

\subsubsection{Comparison of The Iterative Projection and Gradient-based Methods} \label{sec:compareMLIP}
Given current parameters $\xi^i$ of RBM and samples $\underline{x}$ that generated from the underlying distribution $q(x)$, IP could be implemented in two phases:
\begin{enumerate}
  \item In the first phase, we generate samples for the projection $\Gamma_H(p_i)$ of the stationary distribution $p_i(x,h;\xi^i)$ on $H_q$. This is done by sampling $\underline{h}$ from RBM's conditional distribution $p_i(h|x;\xi^i)$ given $\underline{x}$, and hence $(\underline{x},\underline{h})\sim \Gamma_H(p_i(x,h;\xi^i))$;
  \item In the second phase, we train a new RBM with those generated samples $(\underline{x},\underline{h})$, and then update the RBM's parameters to be the newly trained ones, denoted as $\xi^{i+1}$.
\end{enumerate}
 Note that in the second phase all hidden units in RBM are visible in samples $(\underline{x},\underline{h})$. Thus this sub-learning task is similar with training a BM without hidden units, which can be implemented by traditional gradient-based methods.

Given current parameters $\xi^i$ of RBM (with stationary distribution $p_i$) and samples $\underline{x}\sim q(x)$, we can see that both ML and IP share the same sampling process, sampling $(\underline{x},\underline{h})$ in the $\Gamma_H(p_i)$ projection phase of IP and the positive phase of ML.
In terms of the quality of the sampling process, if $\underline{x}$ is sufficient and so is $(\underline{x}, \underline{h})$, both the $\Gamma_H(p_i)$ of IP and positive phase of ML can achieve an accurate estimation of $q(x,h;\xi^i)$. On the other hand, if $\underline{x}$ is insufficient (it is usually true in real-world applications), sampling biases with respect to $q(x,h;\xi^i)$ may be introduced in the sampling process so that the accurate estimation can not be guaranteed.

However, the updating rule is different:
IP realizes the parameter updating by using a sub-learning task (fitting a new RBM to the generated sample $(\underline{x},\underline{h})$), while ML adjusts the parameters directly using Equation (\ref{eq:learnrulestochasticgradient}).
Let $q_{i+1}$ denote the distribution of $(\underline{x},\underline{h})$, and $p_{i+1}$ and $p'_{i+1}$ denote the stationary distributions of RBM after the parameter updating using IP and ML respectively.
With a proper learning rate $\lambda$, the parameter updating phase of ML would lead to the decrease of divergence, that is, $D[q_{i+1},p_{i}] \geq D[q_{i+1},p'_{i+1}]$. Since $p_{i+1}$ is the projection of $q_{i+1}$ on $B$, then we have $D[q_{i+1},p'_{i+1}] \geq D[q_{i+1},p_{i+1}]$. Therefore, ML can be seen as an ``unmature projection'' of $q_{i+1}$ on $B$, and it does not guarantee that the theoretical projection $\Gamma_B(q_{i+1})$ is reached.
To achieve the same projection point $\Gamma_B(q_{i+1})$ as IP, ML needs multiple updating iterations, where each iteration moves the current distribution towards $\Gamma_B(q_{i+1})$ in the gradient direction by some oracle step size (controlled by the learning rate).

Another big difference is that IP separates the positive sampling process and the gradient estimation in two phases: $\Gamma_H$ and $\Gamma_B$, meaning that there is no positive sampling in the sub-learning of $\Gamma_B$. However, ML needs to constantly adjust the gradient direction with respect to certain learning rate immediately after each sampling process. Later experiments in Section \ref{sec:experimentrbm} indicate that this may give IP the advantage of robustness against sampling biases, especially when the gradient is too small to be distinguishable from these biases in the learning process.

\subsubsection{Discussions on deep Boltzmann machine}\label{sec:discussiondeeplearning}
For deep Boltzmann machine (DBM) \citep{Salakhutdinov2012}, several layers of RBM compose a deep architecture in order to achieve a representation at a sufficient abstraction level, where the hidden units are trained to capture the dependencies of units at the lower layers, as shown in Figure \ref{fig:deeparchitecture}. In this section, we give a discussion on some theoretical insights on the deep architectures, in terms of the CIF principle.
\begin{figure}
  \centering
  \includegraphics[width=0.8\textwidth]{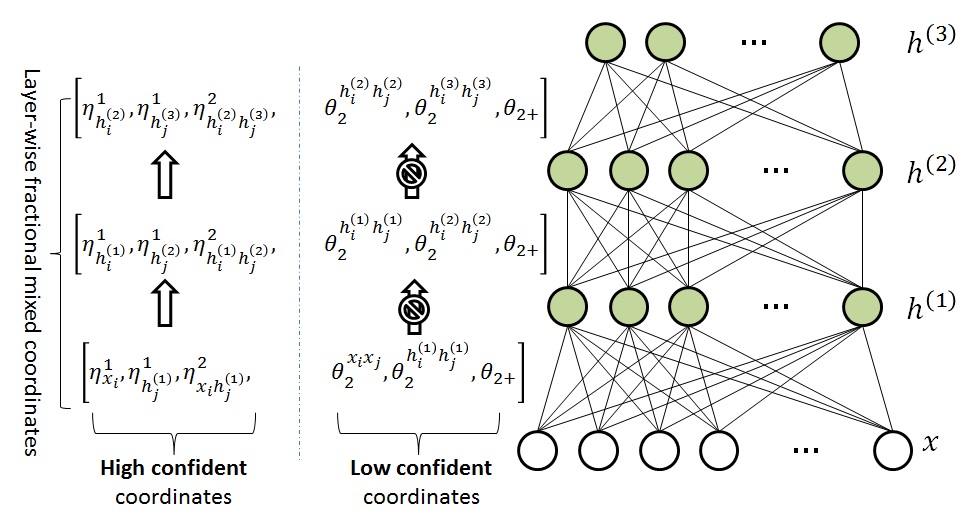}
  \caption{A multi-layer BM with visible units $x$ and hidden layers $h^{(1)}$, $h^{(2)}$ and $h^{(3)}$. The greedy layer-wise training of deep architecture is to maximally preserve the confident information layer by layer. Note that the prohibition sign indicates that the Fisher information on lowly confident coordinates is not preserved.}
  \label{fig:deeparchitecture} 
\end{figure}

In Section \ref{sec:rbmequivalentcc}, we have shown that the structure of RBM implies the \emph{compactness} and \emph{completeness} conditions, which could guide the learning of a good representation. Thus DBM can be seen as the composition of a series of representation learning stages. Then, an immediate question is: what kind of representation of the data should be generated as the output of each stage? From an information abstraction point of view, each stage of the deep architecture could build up more abstract features by using the highly confident information on parameters (or coordinates) that is transmitted from less abstract features in lower layers. Those more abstract features would potentially have a greater representation power \citep{Bengio12representationlearning}. The CIF principle describes how the information flows in those representation transformations, as illustrated in Figure \ref{fig:deeparchitecture}. We propose that each layer of DBM determines a submanifold $M$ of $S$, where $M$ could maximally preserve the highly confident information on parameters, as shown in Section \ref{sec:iplearning}. Then the whole DBM can be seen as the process of repeatedly applying CIF in each layer, achieving the tradeoff between the abstractness of representation features and the intrinsic information confidence preserved on parameters.

Once a good representation has been found at each level by layer-wise application of unsupervised greedy pre-training, it can be used to initialize and train the deep neural networks through supervised learning \citep{hinton06,bengio2010why}. Recall that the straightforward application of gradient-based methods to train all layers of a DBM simultaneously tends to fall into a poor local minima \citep{Younes98onthe,jointtraindbm2012}. Now, our next question is: why can  the layer-wise pre-training give a more reasonable parameter initialisation? Empirically, \citet{bengio2010why} shows that the unsupervised pre-training acts as a regularization on parameters in a way that the parameters are set into a region, from which better basins of attraction can be reached. Theoretically, by using the fractional mixed coordinates, it can be shown that this regularized region is actually the layer-wise restriction using CIF, i.e., the highly confident coordinates are preserved with respect to the target density and all lowly confident coordinates are set to a neutral value of zero, as illustrated in Figure \ref{fig:deeparchitecture}. Effectively, the parameter space is regularized to fall into a region where the parameters can be confidently estimated based on the given data. Under this CIF-based regularization, the pre-training can be seen as searching for a reasonable parameter setting, from which a good representation of the input data can be generated in each layer.

\section{Experimental Study} \label{sec:cd-cif}
In this section, we will empirically investigate the CIF principle in density estimation tasks on two types of Boltzmann machines, i.e., SBM and RBM. More specifically, for SBM, we will investigate how to use CIF to take effect on the learning trajectory with respect to the specific sample, and hence further confine the parameter space to the region corresponding to the most confident information contained in given data. For RBM, it is inconvenient to use a sample-specific strategy since the information of hidden variables is missed. Alternatively we investigate the potential of the iterative projection procedure proposed previously. For both SBM and RBM, two baseline learning methods, i.e., the contractive divergence (CD) \citep{Hinton02cd,hinton05CD} and maximum-likelihood (ML) \citep{ackley85BM}, are adopted.

The ML learning is described in Section \ref{sec:learningMLforBM}. The CD can be seen as an approximation of ML.
Let $q(x)$ be the underlying probability distribution from which sample $\underline{x}$ are generated independently. Then our goal is to train a BM (with stationary probability $p(x,h;\xi)$ in Equation \ref{eq:probBM}) based on $\underline{x}$ that realizes $q(x)$ as faithfully as possible.
Comparing to ML, the CD learning realizes the gradient descend of a different objective function to avoid the difficulty of computing the log-likelihood gradient in ML, shown as follows:
\vspace{-2mm}
\begin{eqnarray}\label{CD-learning}
  \Delta \xi\!\!&=&\!\!-\varepsilon \cdot \frac{\partial (D(q_0, p) - D(p_m, p))}{\partial \xi} = \varepsilon \cdot (-\langle \frac{\partial E(x,h;\xi)}{\partial \xi} \rangle_0 + \langle \frac{\partial E(x,h;\xi)}{\partial \xi} \rangle_m)
\end{eqnarray}
where $q_0$ is the sample distribution, $p_m$ is the distribution by starting the Markov chain with the data and running $m$ steps, $\langle\cdot \rangle_0$ and $\langle\cdot \rangle_m$ denote the averages with respect to the distribution $p_0$ and $p_m$, and $D(\cdot, \cdot)$ denotes the K-L divergence.

In these experiments, two kinds of binary datasets are used:
\begin{enumerate}
  \item The artificial binary dataset: we first randomly select the target distribution $q(x)$, which is chosen uniformly from the open probability simplex over the $n$ random variables. Then, the dataset with $N$ samples are generated from $q(x)$.
  \item The \emph{20 News Groups} binary dataset: \emph{20 News Groups} is a collection of approximately 20,000 newsgroup documents, partitioned evenly across 20 different newsgroups \footnote{The \emph{20 News Group} dataset is freely downloadable from http://qwone.com/$\sim$jason/20Newsgroups/}. The collection is preprocessed using porter stemmer and stop-word removal. We select top 100 terms with highest frequency in the collection. Each document is represented as a 100-dimensional binary vector, where each element indicates whether certain term occurs in current document or not.
\end{enumerate}

\subsection{Experiments with SBM}
From the perspective of IG, we can see that ML/CD learning is to update parameters in SBM so that its corresponding coordinates $[\eta^{2-}]$ are getting closer to the data distribution. This is consistent with our theoretical analysis in Section \ref{sec:cif} and Section \ref{sec:geometryBM} that SBM uses the most confident information (i.e., $[\eta^{2-}]$) for approximating an arbitrary distribution in an expected sense.
But, for the distribution with specific samples, \emph{can CIF further recognize less-confident parameters in SBM and reduce them properly}?

Our solution here is to apply CIF to take effect on the learning trajectory with respect to specific samples, and hence further confine the parameter space to the region that indicated by the most confident information contained in the samples. This experiment shows that given specific samples we need to preserve the confident parameters to certain extend, and there should exist some golden ratio that would produce best performance on average.
\subsubsection{A Sample-specific CIF-based CD Learning} \label{subsec:cifcd}
The main modification of our CIF-based CD algorithm (CD-CIF for short) is that we generate the samples for $p_m(x)$ based on those parameters with confident information, where the confident information carried by certain parameter is inherited from the sample and could be assessed using its Fisher information computed in terms of the sample.

For CD-1 (i.e., $m$=1), the firing probability for the $i$th neuron after one-step transition from the initial state $x^{(0)}=\{x_1^{(0)}, x_2^{(0)}, \dots, x_n^{(0)}\}$) is:
\vspace{-2mm}
\begin{equation}\label{firingProb}
  p_{(m=1)}(x_i^{(1)}=1|x^{(0)}) = \frac{1}{1+exp\{-\sum_{j\neq i}U_{ij}x^{(0)}_j-b_i\}}
\end{equation}

For CD-CIF, the firing probability in Equation \ref{firingProb} is modified as follows:
\vspace{-2mm}
\begin{equation}\label{firingProbCIF}
  p'_{(m=1)}(x_i^{(1)}=1|x^{(0)}) = \frac{1}{1+exp\{-\sum_{(j\neq i) \& (F(U_{ij})>\tau) }U_{ij}x^{(0)}_j-b_i\}}
\end{equation}
where $\tau$ is a pre-selected threshold, $F(U_{ij})=E_{q_0}[x_i x_j]-E_{q_0}[x_i x_j]^2$ is the Fisher information of $U_{ij}$ (see Equation \ref{eq:proposiationFisherMetrictheta}) and the expectations are estimated from the given sample $\underline{x}$. We can see that those weights whose Fisher information are less than $\tau$ are considered to be unreliable w.r.t. $\underline{x}$. In practice, we could setup $\tau$ by the ratio $r$ to specify the remaining proportion of the total Fisher information $T_{FI}$ of all parameters, i.e., $\tau = r*T_{FI}$.

In summary, CD-CIF is realized in two phases. In the first phase, we initially ``guess'' whether certain parameter could be faithfully estimated based on the finite sample. In the second phase, we approximate the gradient using the CD scheme, except for the CIF-based firing function is used.

\subsubsection{Results and Discussions on Artificial Dataset}\label{sec:experiment}
In this section, we empirically investigate our justifications for the CIF principle, especially how the sample-specific CIF-based CD learning works in the context of density estimation.
%

\begin{figure*}
  \centering
  \subfigure[]{
    \label{fig:cifSample} 
    \includegraphics[width=1.4in]{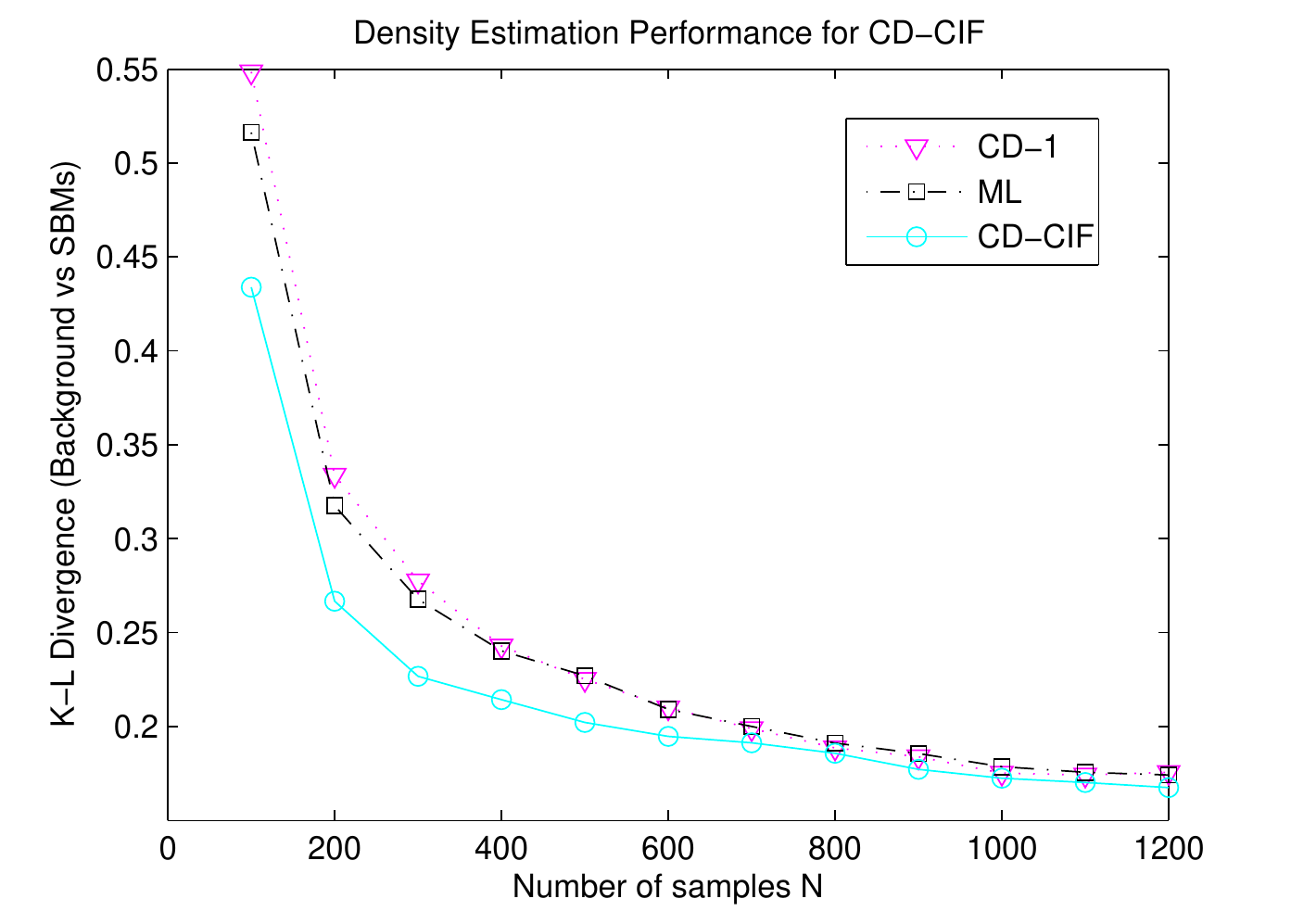}}
     \hspace{-4mm}
  \subfigure[]{
    \label{fig:cif100} 
    \includegraphics[width=1.4in]{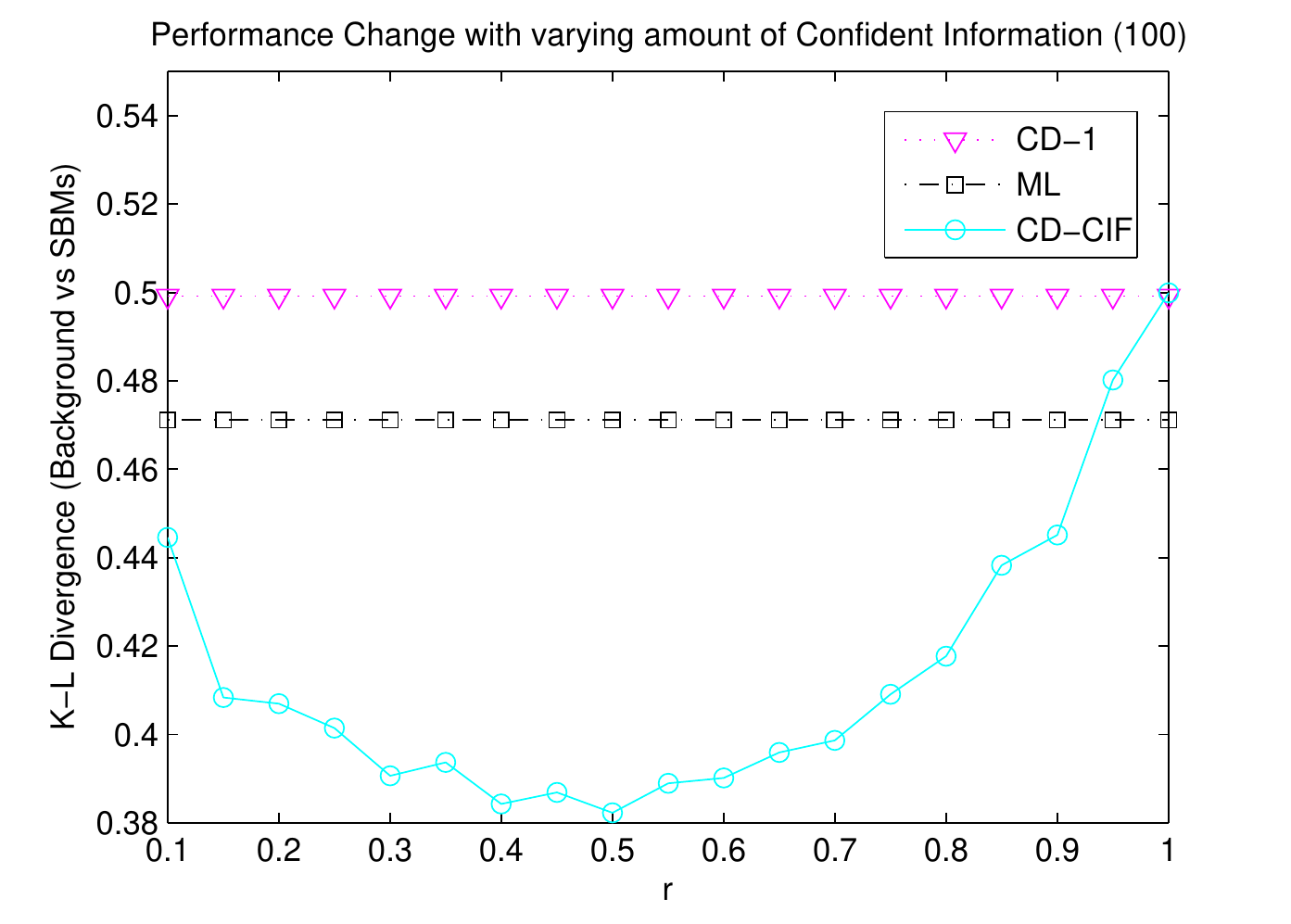}}
    \hspace{-4mm}
  \subfigure[]{
    \label{fig:cif1200} 
    \includegraphics[width=1.4in]{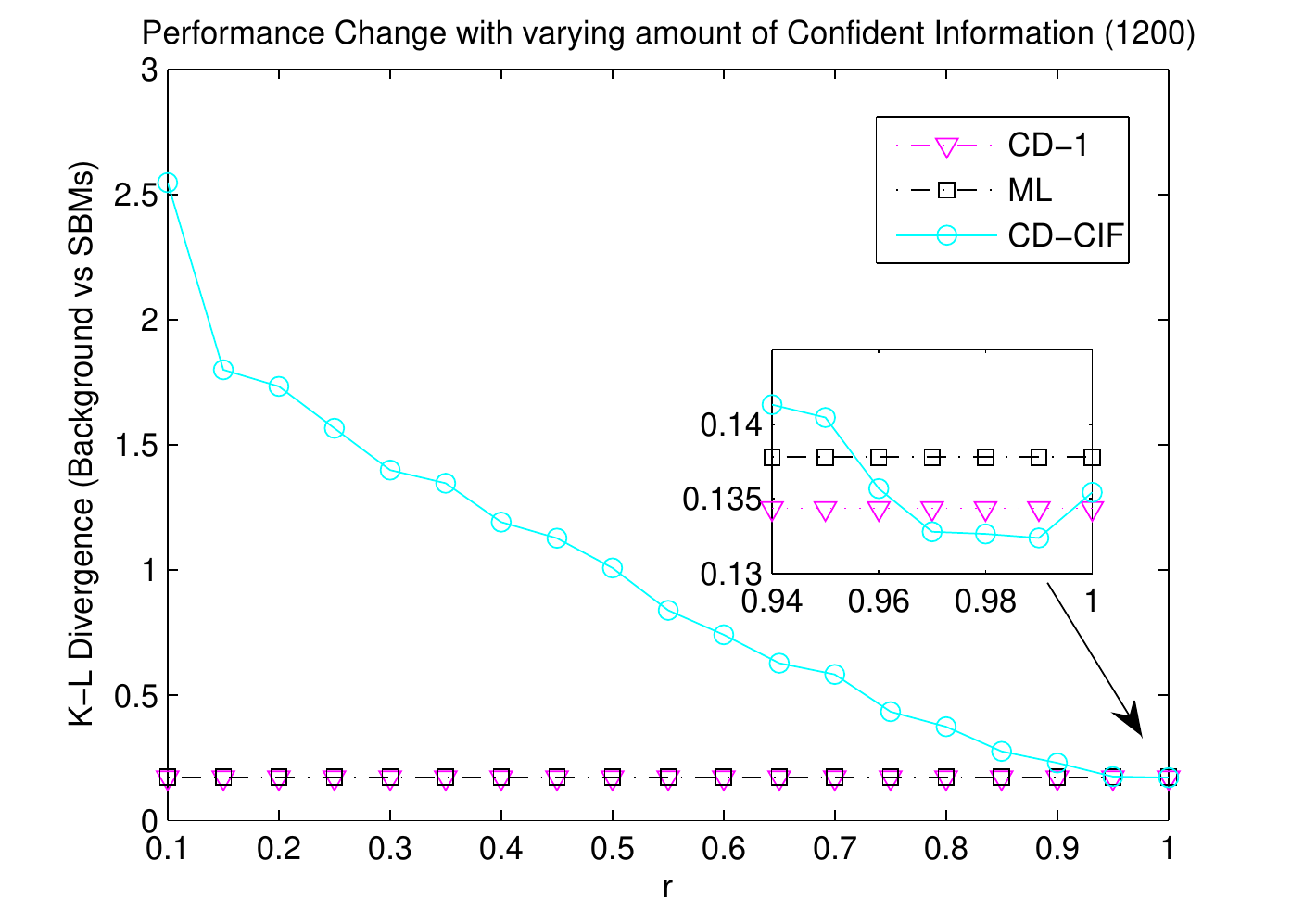}}
    \hspace{-4mm}
  \subfigure[]{
    \label{fig:ciftrajfinal} 
    \includegraphics[width=1.4in]{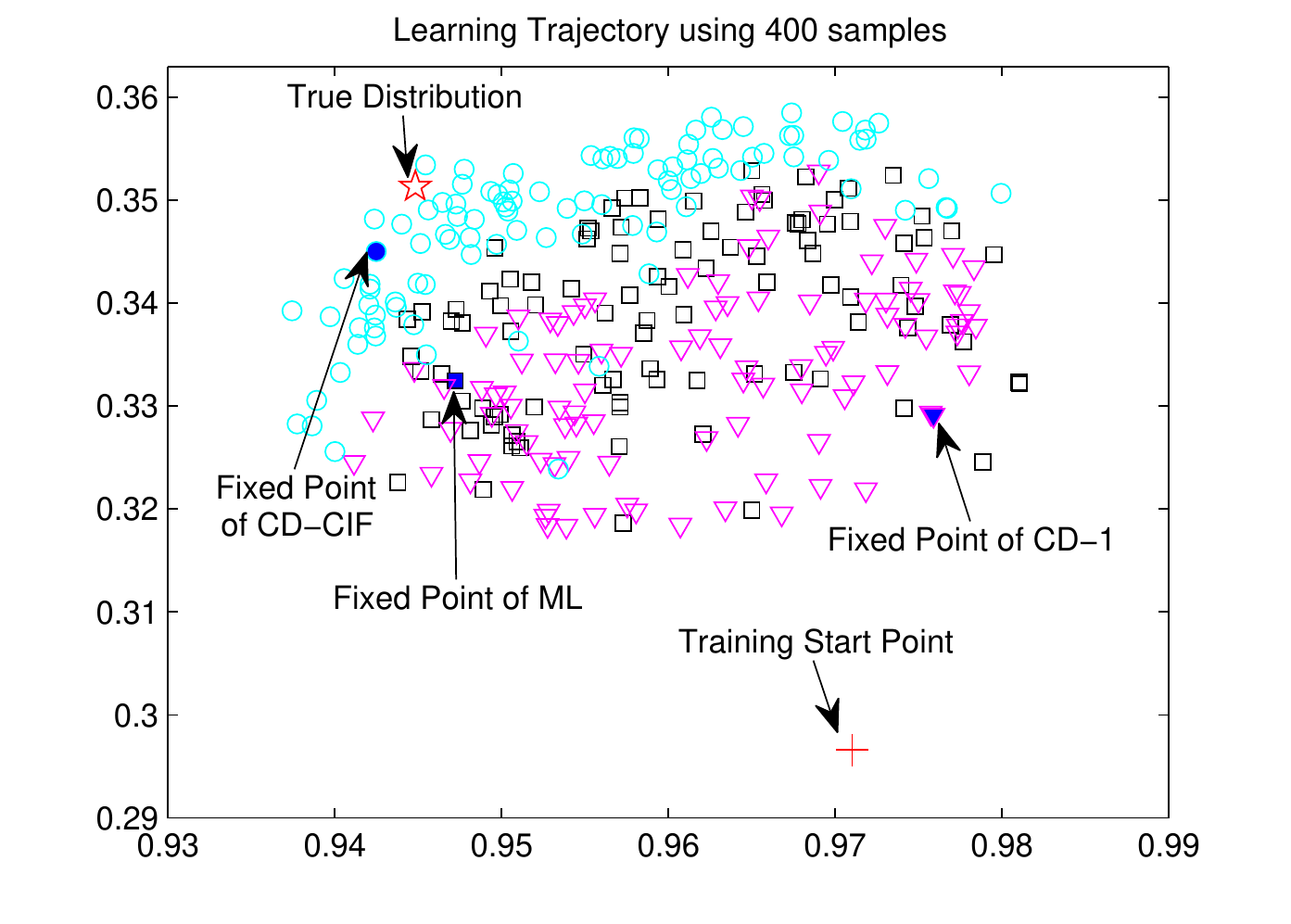}}

  \caption{\subref{fig:cifSample}: the performances of CD-CIF on different sample sizes; \subref{fig:cif100} and \subref{fig:cif1200}: the performances of CD-CIF with various values of $r$ on two typical sample sizes, i.e., 100 and 1200; \subref{fig:ciftrajfinal}: learning trajectories of last 100 steps for ML (squares), CD-1 (triangles) and CD-CIF (circles).}
  \label{fig:whole} 
\end{figure*}

\textbf{Experimental Setup and Evaluation Metric}:
For computation simplicity, the artificial dataset is set to be 10-dimensional. Three learning algorithms are investigated: ML, CD-1 and our CD-CIF. K-L divergence is used to evaluate the goodness-of-fit of the SBM trained by various algorithms. For sample size $N$, we run 100 instances (20 randomly generated distributions $\times$ 5 randomly running) and report the averaged K-L divergences.
Note that we focus on the case that the variable number is relatively small ($n=10$) in order to analytically evaluate the K-L divergence and give a detailed study on algorithms. Changing the number of variables only offers a trivial influence for experimental results since we obtained qualitatively similar observations on various variable numbers (not reported here).

\textbf{Automatically Adjusting $r$ for Different Sample Sizes}: The Fisher information is additive for i.i.d. sampling. When sample size $N$ changes, it is naturally to require that the total amount of Fisher information contained in all tailored parameters is steady. Hence we have $\alpha = (1-r)N$, where $\alpha$ indicates the amount of Fisher information and becomes a constant when the learning model and the underlying distribution family are given. It turns out that we can first identify $\alpha$ using the optimal $r$ w.r.t. several distributions generated from the underlying distribution family, and then determine the optimal $r$ for various sample sizes using: $r = 1 - \alpha / N$. In our experiments, we set $\alpha = 35$.

\textbf{Density Estimation Performance}: The averaged K-L divergences between SBM (learned by ML, CD-1 and CD-CIF with the $r$ automatically determined) and the underlying distribution are shown in Figure \ref{fig:cifSample}. In the case of relatively small samples ($N \leq 500$) in Figure \ref{fig:cifSample}, our CD-CIF method shows significant improvements over ML (from 10.3\% to 16.0\%) and CD-1(from 11.0\% to 21.0\%).
This is because we could not expect to have reliable identifications for all model parameters from insufficient samples, and hence CD-CIF gains its advantages by using parameters that could be confidently estimated.
This result is consistent with our previous theoretical insight that Fisher information gives a reasonable guidance for parametric reduction via the confidence criterion.
As the sample size increases ($N\geq 600$), CD-CIF, ML and CD-1 tend to have similar performances. Since with relatively large samples most model parameters can be reasonably estimated, and hence the effect of parameter reduction using CIF gradually becomes marginal. In Figure \ref{fig:cif100}, Figure \ref{fig:cif1200}, we show how sample size affects the interval of $r$ that achieves improvements over CD-1. For $N=100$, CD-CIF achieves significantly better performances for a wide range of $r$. While, for $N=1200$, CD-CIF can only marginally outperform baselines for a narrow range of $r$.

\textbf{Effects on Learning Trajectories}: We use the 2D visualizing technology SNE to investigate learning trajectories and dynamical behaviors of three comparative algorithms \citep{hinton05CD}. We start three methods with the same parameter initialization. Then each intermediate state is represented by a 55-dimensional vector formed by its current parameter values. From Figure \ref{fig:ciftrajfinal}, we can see that:
1) In the final 100 steps, three methods seem to end up with staying in different regions of the parameter space, and CD-CIF confines the parameter in a relatively thinner region compared to ML and CD-1;
2) The true distribution is usually located on the side of CD-CIF, indicating its potential for converging to the optimal solution.
Note that the above claims are based on general observations and Figure \ref{fig:ciftrajfinal} is shown as an illustration.
Hence we may conclude that CD-CIF regularizes the learning trajectories in a desired region of the parameter space using the sample-specific CIF principle.
\begin{figure}
  \centering
  \includegraphics[width=0.6\textwidth]{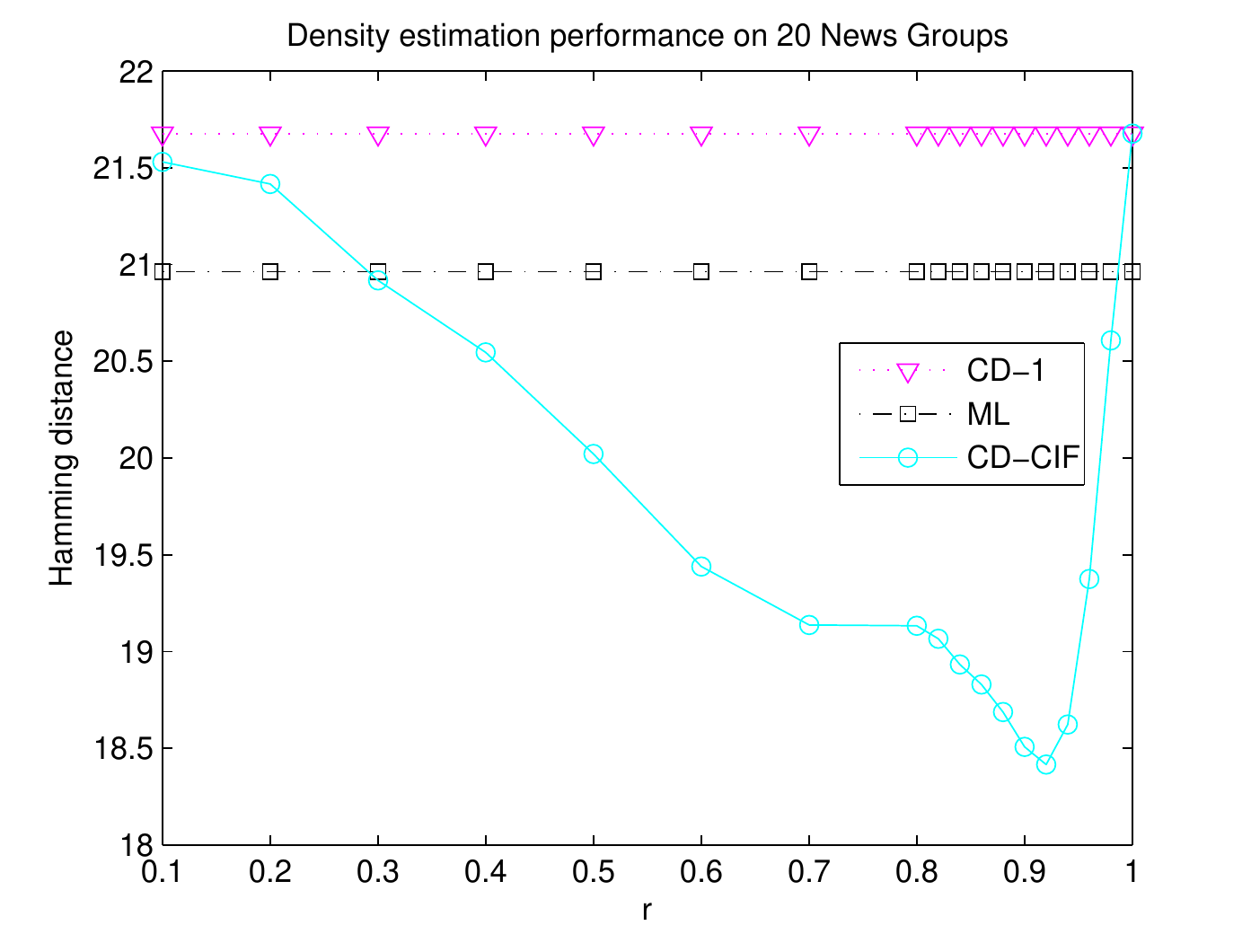}
  \caption{the performances of CD-CIF with various values of $r$ on \emph{20 news group} dataset}
  \label{fig:20newsgroup} 
\end{figure}
\subsubsection{Results and Discussions on Real Textual Dataset}\label{sec:experimentsbmtextual}
In this section, we empirically investigate how the sample-specific CIF-based CD learning works on real-world datasets in the context of density estimation. In particular, we use the SBM to learn the underlying probability density over 100 terms of the \emph{20 News Groups} binary dataset. The learning rate for CD-1, ML and CD-CIF are manually tuned in order to converge properly and all set to 0.001. Since it is infeasible to compute the K-L devergence due to the high dimensionality, the averaged Hamming distance between the samples in the dataset and those generated from the SBM is used to evaluate the goodness-of-fit of the SBM's trained by various algorithms. Let $D=\{d_1, d_2,\dots, d_N\}$ denote the dataset of $N$ documents, where each document $d_i$ is a 100-dimensional binary vector. To evaluate a SBM with parameter $\xi_{sbm}$, we first randomly generate $N$ samples from the stationary distribution $p(x;\xi_{sbm})$, denoted as $V=\{v_1, v_2,\dots, v_N\}$. Then the averaged hamming distance $D_{ham}$ is calculated as follows:
$$D_{ham}[D,V]=\frac{\sum_{d_i} (\min_{v_j} (Ham[d_i, v_j])}{N}$$
where $Ham[d_i, v_j]$ is the number of positions at which the corresponding values are different.

The result is shown in Figure \ref{fig:20newsgroup}. Our CD-CIF method shows maximal improvements over ML (12.15\%) and CD-1(15.05\%) at $r=0.92$. We can also see that CD-CIF achieves significantly better performances for a wide range of $r \in [0.5,0.96]$, which is consistent with our observations with the experiments on artificial datasets when the samples is insufficient.

\subsection{Experiments with RBM} \label{sec:experimentrbm}
The RBM is practically more interesting than SBM, since it has higher representational power.
In this section, we will compare three different learning algorithms for RBM: CD-1, ML and IP.
In \citet{hinton05CD}, CD is shown to be biased with respect to ML for almost all data distributions. In Section \ref{sec:compareMLIP}, we have compared ML and IP theoretically. In this section, an empirical study on the three algorithms is conducted.
\subsubsection{Results and Discussion on Artificial Datasets} \label{sec:resultanalysisip}

\textbf{Experimental Setup and Evaluation Metric}: For computational simplicity, the artificial dataset is of 5 dimensionality, and the number of hidden units in RBM is set to 5. Three learning algorithms are investigated: ML, CD-1 and IP. K-L divergence is used to evaluate the goodness-of-fit of the RBM's trained by various algorithms. Six different sample sizes $N$ are tested, namely 50, 100, 200, 500, 5000 and 50000. For sample size $N$, the learning rates for CD-1 and ML are set to be $\varepsilon = 0.5/N$, and we observe that they could converge properly. For the sub-learning phase $\Gamma_B$ of IP, we adopt the CD algorithm for the training of BM without hidden units, whose learning rate is also set to $\varepsilon = 0.5/N$. We need to scan the dataset multiple times in order to iteratively update parameters, and each full scan of the whole dataset is called an epoch. In CD-1 and ML, we set the maximal number of epoches to 8000. We run the IP for 40 iterations, and each iteration is a CD sub-training with the maximal number of epoches setting to 200. We adopt CD-1 as the baseline method.

\textbf{Results and Analysis}:
\begin{figure*}
  \centering
  \subfigure[]{
    \label{fig:klperformanceIP100} 
    \includegraphics[width=0.47\textwidth]{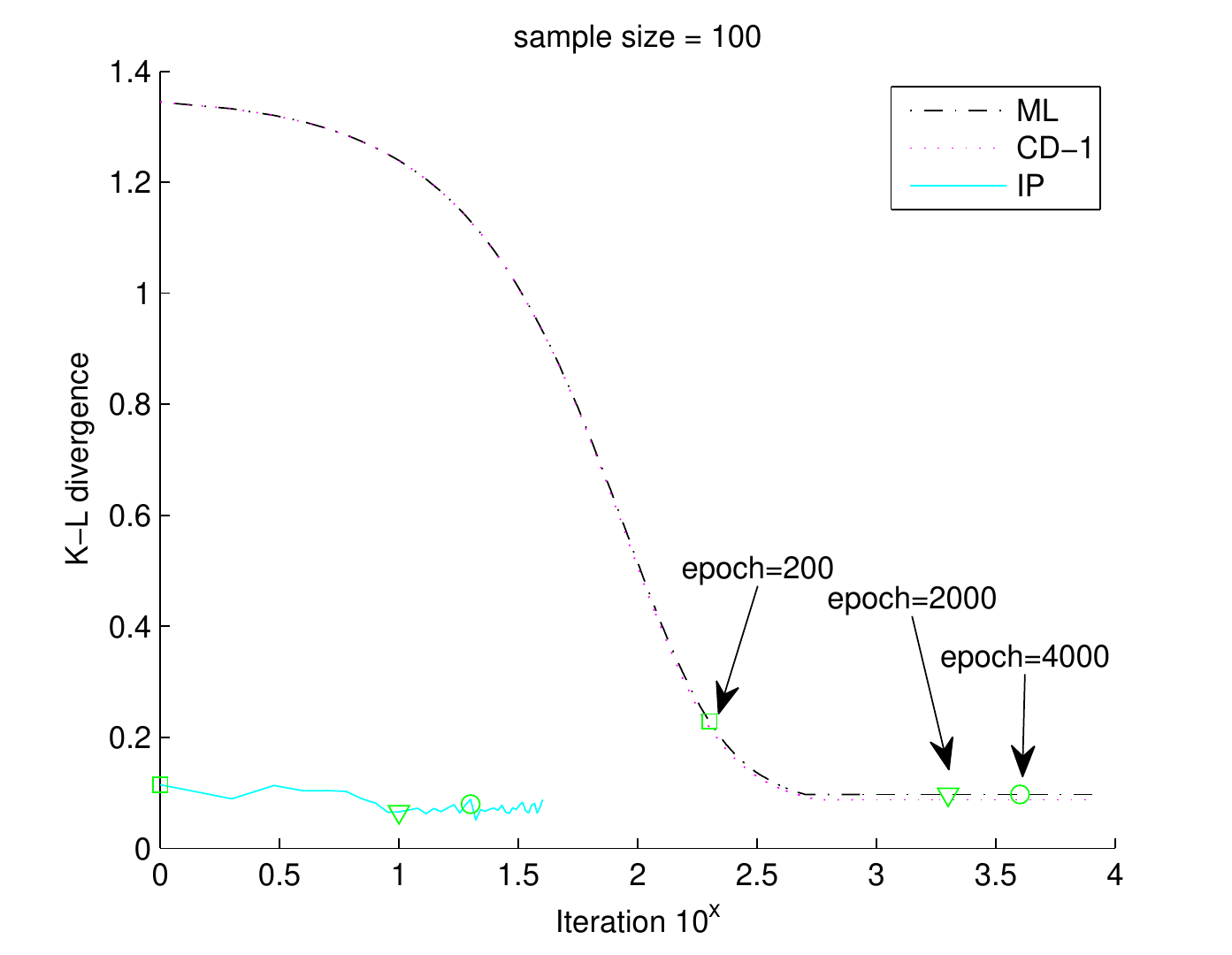}}
     \hspace{-4mm}
  \subfigure[]{
    \label{fig:klperformanceIP50000} 
    \includegraphics[width=0.47\textwidth]{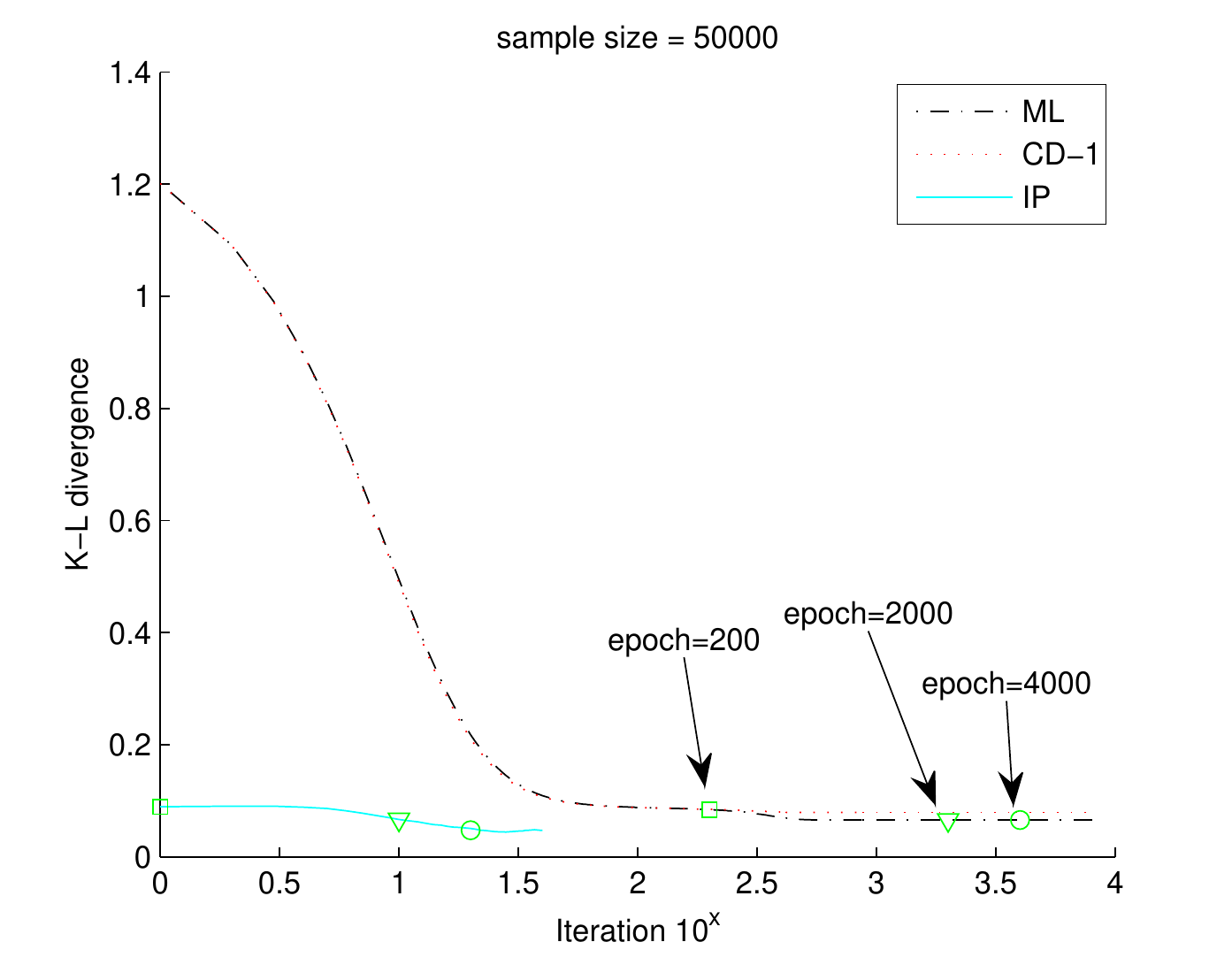}}

  \caption{\subref{fig:klperformanceIP100} and \subref{fig:klperformanceIP50000} illustrate the averaged learning curves for CD-1, ML and IP for 30 randomly chosen data distributions, with sample size 100 and 50000 respectively. The x-axis is in $\log_{10}$ scale. To compare the efficiency of different algorithms along the time-line (the time unit is an epoch, that is, the time used for the parameter updates in each full scan of the whole dataset), we plot three time stamps (epoch=100, 1000 and 4000). Note that each iteration in IP contains a sub-training task, which is trained using 200 epoch in this experiment.}
  \label{fig:whole} 
\end{figure*}
The on-average performances of the three methods on dataset of different scales are shown in Table \ref{tab:resultperformanceip} in the context of density estimation.
In order to study the behavior of IP, we plot the sequences of K-L divergence between target distribution and RBM in each iteration along the whole learning trajectory, shown in Figure \ref{fig:whole}. Comparing CD-1 with ML, we can see that the K-L divergences of both CD-1 and ML decrease in a similar way, converging at the same rate, taking the same number of iterations to converge to a given tolerance, which is consistent with the conclusion in \cite{Hinton02cd}.

From the convergence behavior of IP shown in Figure \ref{fig:whole}, we can see that the general trend is that the K-L divergence decreases steadily with small fluctuations. Since there are only few iterations in IP, it is reasonable for us to select the best performance that IP has reached in the whole learning process, called the best IP. Here the K-L divergence between the sample distribution and RBM's stationary distribution is adopted as the selection metric. Thus, in addition to the converging performance, we also show the best performance selected among all 40 iterations for IP in Table \ref{tab:resultperformanceip}. Note that the best performances for CD-1 and ML are not reported, since their converging performances are often approximately the best ones.

For small sample size (e.g., 50 and 100), we can see that the converging performance of IP is comparable with respect to CD and ML. As the sample size increases, we can see that IP gradually outperforms CD-1 and ML, and shows significant improvement on large sample size (e.g., 5000 and 50000). For the best IP, its performance is significantly better than CD-1 and ML for all sample sizes, indicating that IP has the potential to further improve its performance by using some suitable selection metric.
Comparing to ML, IP takes much shorter iterations as expected to achieve a performance threshold.
Theoretically, IP can converge to the local minimum of RBM based on our theoretical analysis in Section \ref{sec:compareMLIP}. If a proper learning rate is selected, ML can also converge to a local minimum.
However, one interesting result is that sometimes there is a big difference between the convergence points of IP and ML, even in cases that the sampling is sufficient (e.g., sample size equals 50000).

This can be explained as follows. In practice, ML needs to constantly do positive and negative sampling in each updating, which may produce much sample biases. As the gradient decreasing to a small value, the correct gradient direction may be fluctuated by the sample biases. Thus, instead of converging to the local minimum, ML might fluctuate around some sub-optimal region. Actually, the $\Gamma_H$ in IP is a sampling process that may also introduce sample biases. That is why the IP is fluctuated when the sample is insufficient, as shown in Figure \ref{fig:klperformanceIP100}. Given sufficient samples, the sampling biases in $\Gamma_H$ decreases and the fluctuation of IP declines, as shown in Figure \ref{fig:klperformanceIP50000}. For ML, though the sampling biases for the whole dataset becomes small given sufficient samples, the gradient estimation for each sample is still closely intergraded with the sample bias of certain sample, meaning that this inseparable coupling relationship results in a biased gradient estimation. The main advantage of the IP procedure over traditional gradient-based methods is the separation of the positive sampling process and the gradient estimation. We conjecture that this independent design gives IP the potential to achieve optimal solutions that CD-1 and ML can not reach.

\begin{table}
\begin{center}
\begin{tabular}{|c|c|c|c|c|}
  \hline
  Sample Size & CD-1 & ML & Converge IP (-chg\%)& Best IP (-chg\%) \\
  \hline
50  & 0.0970 & 0.0971 & 0.0978 (+0.81\%) & 0.0774 (-20.21\%) \\
100  & 0.0786 & 0.0793 & 0.0787 (+0.08\%) & 0.0637 (-18.92\%) \\
200  & 0.0672 & 0.0678 & 0.0640 (-4.75\%) & 0.0575 (-14.45\%) \\
500  & 0.0621 & 0.0621 & 0.0594 (-4.25\%) & 0.0567 (-8.61\%) \\
5000  & 0.0532 & 0.0524 & 0.0468 (-12.00\%) & 0.0437 (-17.81\%) \\
50000  & 0.0497 & 0.0475 & 0.0411 (-17.37\%) & 0.0332 (-33.08\%) \\
  \hline
\end{tabular}
\end{center}
\caption{Performance comparison of CD-1, ML and IP in the density estimation task. The change of IP with respect to the baseline method (CD-1) is reported.}
\label{tab:resultperformanceip}
\end{table}

\subsubsection{Results and Discussion on Real Textual Dataset}
\begin{figure}
  \centering
  \includegraphics[width=0.6\textwidth]{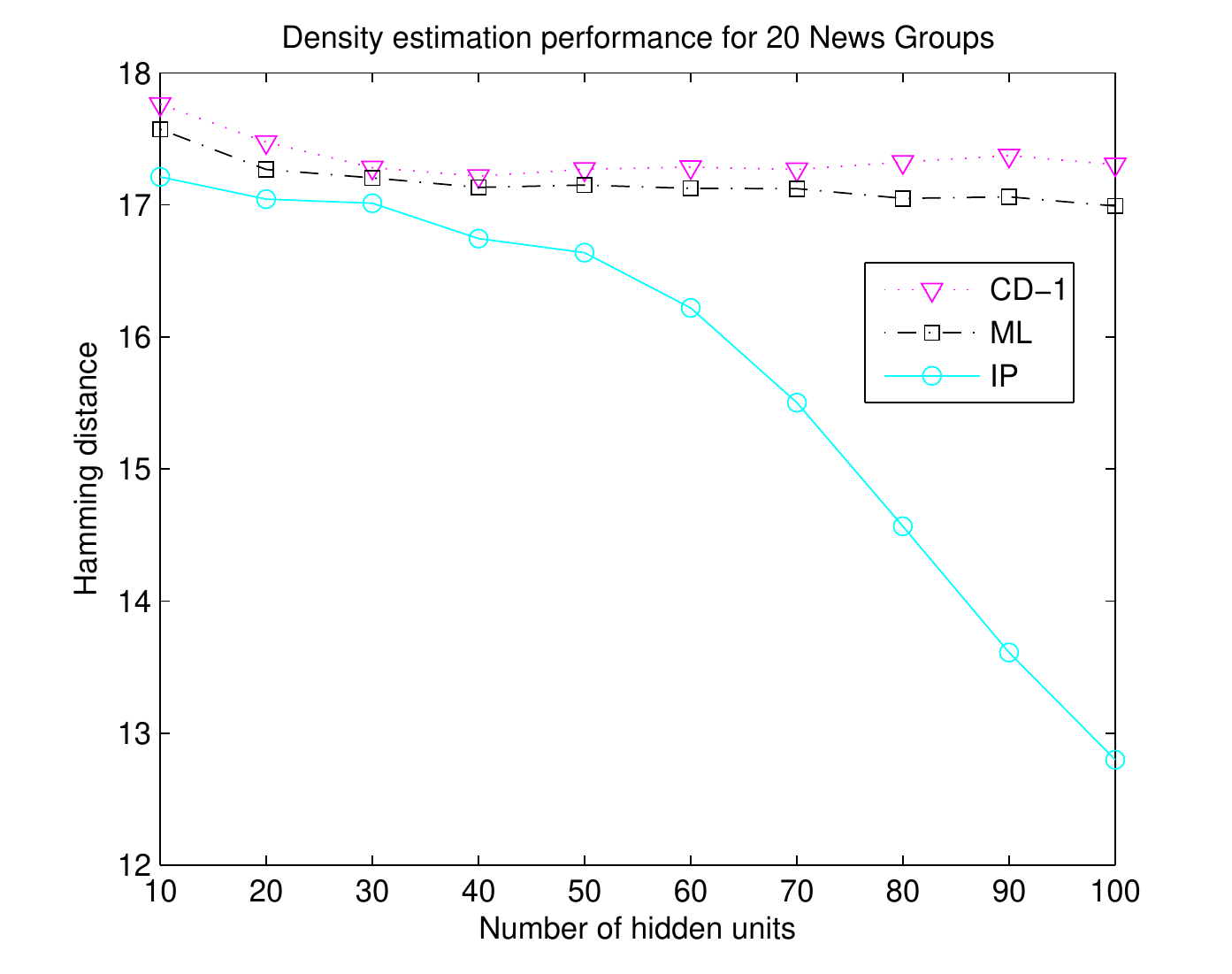}
  \caption{Performances of IP with various number of hidden units on \emph{20 News Group}}
  \label{fig:ip20newsgroup} 
\end{figure}

In this section, we empirically investigate how the IP works for RBM on real-world datasets in the context of density estimation. We use the RBM to learn the probability density over 100 terms on the \emph{20 News Groups} binary dataset. The number of hidden units $n_h$ in RBM is set to $[10, 20, \dots, 100]$. The learning rate for CD-1and ML are manually tuned in order to
converge properly and all set to 0.001. The learning rate for the CD sub-learning task in IP is also set to 0.001. We run IP for 9 iterations for all settings of $n_h$. Similar with the experiments in Section \ref{sec:experimentsbmtextual}, the averaged Hamming distance is used to evaluate the goodness-of-fit of the RBM's trained by various algorithms.
The average Hamming distances for ML, CD-1 and IP are shown in Figure \ref{fig:ip20newsgroup}. We can see that IP achieves better performances on all settings of $n_h$. And the hamming distance for IP drops dramatically as $n_h$ increases. This trend can be explained as follows. As $n_h$ grows, the sampling biases increases and the interference of sampling biases with respect to the gradient estimation becomes more and more serious. This limits the actual performance of RBMs learnt by CD-1 and ML, with respect to the growing modelling power gained by increasing $n_h$. As shown in Section \ref{sec:resultanalysisip}, the IP procedure separates the positive sampling process and the gradient estimation in two phases: $\Gamma_H$ and $\Gamma_B$. This result shows that IP has the potential to achieve optimal solutions that CD-1 and ML can not reach in real-world applications.

\section{Conclusions} \label{sec:conclusions}
The CIF principle proposed in this paper tackles the problem of dimensionality reduction in parameter space by preserving the parameters with highly confident estimates and tailoring the less confident parameters. It provides a strategy for the derivation of probabilistic models. The SBM and RBM are specific examples in this regard. They have been theoretically shown to achieve a reliable representation in parameter spaces by exactly using the general CIF principle. CIF gives us a principled and context-independent way to address the questions on what we should do for parameter reduction (regularization) and how to do it. Based on CIF, we also show that the deep neural networks consisting of several layers of RBM can be seen as the layer-wise usage of CIF, leading to some theoretical interpretations of the rationale behind deep learning models.

One interesting result shown in our experiments is that: although CD-CIF is a biased algorithm, it could significantly outperform ML when the sample is insufficient. This suggests that CIF gives us a reasonable criterion for recognizing and utilizing confident information from the underlying data while ML fails to do so.
Another interesting observation is that ML and the CIF-based IP lead to different convergence points in the training of RBM. Our experimental results indicate that IP has the advantage of robustness against sampling biases, due to the separation of the positive sampling process and the gradient estimation.

In the future, we will further develop the formal justification of CIF w.r.t various contexts (e.g., distribution families or models). We will also conduct more extensive experiments on real world applications, such as document classification and handwritten digit recognition, to further justify the properties of IP. We will also extend the IP to train deep neural networks.
\acks{This work is partially supported by the Chinese National Program on Key Basic Research Project (973 Program, grant no. 2013CB329304 and 2014CB744604), the Natural Science Foundation of China (grants no. 61070044, 61111130190, 61272265 and 61105072), and the European Union Framework 7 Marie-Curie International Research Staff Exchange Programme (grant no. 247590).}


\appendix
\section{}
\subsection{Proof of Proposition \ref{prop:fishermatrix}} \label{appendix:thetafisher}
\begin{proof}
By definition, we have:
\begin{equation}\label{eq:propFisher1}
    g_{IJ}=\frac{\partial^2 \psi(\theta)}{\partial \theta ^I \partial \theta ^J} \nonumber
\end{equation}
where $\psi(\theta)$ is defined by Equation (\ref{eq:legedre}). Hence, we have:
\begin{eqnarray}\label{eq:propFisher2}
  g_{IJ}=\frac{ \partial^2 (\sum_{I} {\theta^I \eta_I} - \phi(\eta))}{\partial \theta^I \partial \theta^J} = \frac{\partial \eta_I}{\partial \theta^J} \nonumber
\end{eqnarray}
By differentiating $\eta_I$, defined by Equation (\ref{eq:etacoordinate}), with respect to $\theta^J$, we have:
\begin{eqnarray}\label{eq:propFisher3}
  g_{IJ}&=&\frac{\partial \eta_I}{\partial \theta^J} =  \frac{\partial \sum_x X_I(x)(exp\{\sum_I{\theta^I X_I(x)} - \psi(\theta)\})}{\partial \theta^J}\nonumber \\
   &=& \sum_x {X_I(x) [X_J(x) - \eta_J] p(x;\theta)} = \eta_{I\bigcup J}-\eta_I \eta_J \nonumber
\end{eqnarray}
This completes the proof.
\end{proof}
\vspace{-5mm}
\subsection{Proof of Proposition \ref{prop:fishermatrix_eta}} \label{appendix:etafisher}
\begin{proof}
By definition, we have:
\begin{equation}\label{eq:propEtaFisher}
    g^{IJ} = \frac{\partial^2 \phi(\eta)}{\partial \eta_I \partial \eta_J} \nonumber
\end{equation}
where $\phi(\eta)$ is defined by Equation (\ref{eq:legedre}). Hence we have:
\begin{eqnarray}\label{eq:propEtaFisher2}
  g^{IJ} &=& \frac{\partial^2 (\sum_{J} {\theta^J \eta_J} - \psi(\theta))}{\partial \eta_I \partial \eta_J} = \frac{\partial \theta^I}{\partial \eta_J} \nonumber
\end{eqnarray}
Based on Equation (\ref{eq:thetacoordinate}) and (\ref{eq:etacoordinate}), the $\theta^I$ and $p_{K}$ could be calculated by solving a linear equation system of $[p]$ and $[\eta]$ respectively. Hence we have:
\begin{equation}\label{eq:propThetaP}
    \theta^I = \sum_{K \subseteq I} (-1)^{|I-K|} log(p_{K});~~p_{K} = \sum_{K \subseteq J} (-1)^{|J-K|} \eta_{J} \nonumber
\end{equation}
Therefore, the partial derivation of $\theta^I$ with respect to $\eta_J$ is:
\begin{eqnarray}\label{eq:propEtaFisher3}
    g^{IJ}=\frac{\partial \theta^I}{\partial \eta_J} &=&\sum_{K} \frac{\partial \theta^I}{\partial p_{K}} \cdot \frac{\partial p_{K}}{\partial \eta_J} = \sum_{K\subseteq I\cap J}{(-1)^{|I-K|+|J-K|} \cdot \frac{1}{p_{K}}} \nonumber
\end{eqnarray}
This completes the proof.
\end{proof}

\subsection{Proof of Proposition \ref{prop:fishermatrix_mix}} \label{appendix:mixfisher}
\begin{proof}
The Fisher information matrix of $[\zeta]$ could be partitioned into four parts: $G_\zeta= \left(
                     \begin{array}{cc}
                       A & C \\
                        D & B \\
                     \end{array}
                   \right)$.
It can be verified that in the mixed coordinate, the $\theta$-coordinate of order $k$ is orthogonal to any $\eta$-coordinate less than $k$-order, impling the corresponding element of Fisher information matrix is zero ($C=D=0$) \citep{Nakahara02informationgeometric}. Hence, $G_\zeta$ is a block diagonal matrix.

According to Cram\'{e}r-Rao bound \citep{rao45attainable}, a parameter (or a pair of parameters) has a unique asymptotically tight lower bound of the variance (or covariance) of unbiased estimate, which is given by the corresponding element of the inverse of Fisher information matrix involving this parameter (or this pair of parameters). Recall that $I_\eta$ is the index set of the parameters shared by $[\eta]$ and $[\zeta]_l$ and $J_\theta$ is the index set of the parameters shared by $[\theta]$ and $[\zeta]_l$, we have $(G_\zeta^{-1})_{I_\zeta} = (G_\eta^{-1})_{I_\eta}$ and $(G_\zeta^{-1})_{J_\zeta} = (G_\theta^{-1})_{J_\theta}$, i.e.,
$G_{\zeta}^{-1}= \left(
                     \begin{array}{cc}
                       (G_\eta^{-1})_{I_\eta} & 0 \\
                        0 & (G_\theta^{-1})_{J_\theta} \\
                     \end{array}
                   \right)$.
Since $G_\zeta$ is a block tridiagonal matrix, the proposition follows.
\end{proof}
\vspace{-6mm}
\subsection{Proof of Proposition \ref{prop:fishermatrix_mixDiagonal}} \label{appendix:etafisherDiagonal}
\begin{proof}
Assume the Fisher information matrix of $[\theta]$ be
$G_\theta=\left(
            \begin{array}{cc}
              U & X \\
              X^T & V \\
            \end{array}
          \right)
$, which is partitioned based on $I_\eta$ and $J_\theta$. Based on Proposition \ref{prop:fishermatrix_mix}, we have $A=U^{-1}$. Obviously, the diagonal elements of $U$ are all smaller than one. According to the succeeding Lemma \ref{lemma:diagOfmix}, we can see that the diagonal elements of $A$ ($i.e., U^{-1}$) are greater than $1$.

Next we need to show that the diagonal elements of $B$ are smaller than $1$. Using the Schur complement of $G_\theta$, the bottom-right block of $G_\theta^{-1}$, i.e., $(G_\theta^{-1})_{J_\theta}$, equals to $(V-X^TU^{-1}X)^{-1}$. Thus the diagonal elements of B: $B_{jj}=(V-X^TU^{-1}X)_{jj}<V_{jj}<1$. Hence we complete the proof.
\end{proof}
\vspace{-6mm}
\begin{lemma} \label{lemma:diagOfmix}
With a $l\times l$ positive definite matrix $H$, if $H_{ii}<1$, then $(H^{-1})_{ii}>1, \forall i\in \{1,2,\dots,l\}$.
\end{lemma}
\begin{proof}
Since $H$ is positive definite, it is a Gramian matrix of $l$ linearly independent vectors $v_1,v_2,\dots,v_l$, i.e., $H_{ij}=\langle v_i,v_j\rangle$ ($\langle \cdot,\cdot \rangle$ denotes the inner product). Similarly, $H^{-1}$ is the Gramian matrix of $l$ linearly independent vectors $w_1,w_2,\dots,w_l$ and $(H^{-1})_{ij}=\langle w_i,w_j\rangle$. It is easy to verify that $\langle w_i,v_i\rangle=1, \forall i \in \{1,2,\dots,l\}$. If $H_{ii}<1$, we can see that the norm $\|v_i\|=\sqrt{H_{ii}}<1$. Since $\|w_i\| \times \|v_i\|\geq \langle w_i,v_i\rangle=1$, we have $\|w_i\|>1$. Hence, $(H^{-1})_{ii}=\langle w_i,w_i\rangle=\|w_i\|^2 >1$.
\end{proof}
\vspace{-5mm}
\subsection{Proof of Proposition \ref{prop:GeometricView}} \label{appendix:geometriccif}
\begin{proof}
Let $B_q$ be a $\varepsilon$-ball surface centered at $q(x)$ on manifold $S$, i.e., $B_q= \{\zeta \in S |\|\zeta - \zeta_q\|_2 = \varepsilon\}$, where $\|\cdot\|_2$ denotes the Euclid norm and $\zeta_q$ is the coordinates of $q(x)$. Let $q(x)+dq$ be a neighbor of $q(x)$ uniformly sampled on $B_q$ and $\zeta_{q(x)+dq}$ be its corresponding coordinates. For a small $\varepsilon$, we can calculate the expected information distance between $q(x)$ and $q(x)+dq$ as follows:
\begin{equation}\label{eq:expectdistance}
    E_{B_q}=\int  [(\zeta_{q(x)+dq}-\zeta_q)^T G_\zeta (\zeta_{q(x)+dq}-\zeta_q)]^{\frac{1}{2}} dB_q
\end{equation}
where $G_\zeta$ is the Fisher information matrix at $q(x)$.

Since Fisher information matrix $G_\zeta$ is both positive definite and symmetric, there exists a singular value decomposition $G_\zeta = U^T \Lambda U$ where $U$ is an orthogonal matrix and $\Lambda$ is a diagonal matrix with diagonal entries equal to the eigenvalues of $G_\zeta$ (all $\geq 0$).

Apply the singular value decomposition into Equation (\ref{eq:expectdistance}), the distance becomes:
\begin{equation}\label{eq:expectdistance2}
    E_{B_q}\!\!=\!\!\!\!\int \![(\zeta_{q(x)+dq}-\zeta_q)^T U^T \Lambda U (\zeta_{q(x)+dq}-\zeta_q)]^{\frac{1}{2}} dB_q
\end{equation}
Note that $U$ is an orthogonal matrix, and the transformation $U (\zeta_{q(x)+dq}-\zeta_q)$ is a norm-preserving rotation.

Now we need to show that among all tailored $k$-dimensional submanifolds of $S$, $[\zeta]_{l_t}$ is the one that preserves maximum information distance. Assume $I_{T}=\{i_1, i_2, \dots, i_k\}$ is the index of $k$ coordinates that we choose to form the tailored submanifold $T$ in the mixed-coordinates $[\zeta]$. Based on Equation (\ref{eq:expectdistance2}), the expected information distance $E_{B_q}$ for $T$ is proportional to the sum of eigenvalues of the sub-matrix $(G_\zeta)_{I_T}$, where the sum equals to the trace of $(G_\zeta)_{I_T}$.

Next we show that the sub-matrix of $G_\zeta$ specified by $[\zeta]_{l_t}$ gives maximum trace. Based on Proposition \ref{prop:fishermatrix_mixDiagonal}, the elements on the main diagonal of the sub-matrix $A$ are lower bounded by one, and those of $B$ upper bounded by one. Therefore, $[\zeta]_{l_t}$  gives maximum trace among all sub-matrices of $G_\zeta$. This completes the proof.
\end{proof}
\vspace{-5mm}
\subsection{Proof of Proposition \ref{prop:SBMCIF}} \label{appendix:SBMCIF}
\begin{proof}
Let $M_{sbm}$ be the set of all probability distributions realized by SBM. \citet{amari92igbm} proves that the mixed-coordinates of the resulting projection $P$ on $M_{sbm}$ is $[\zeta]_{P}=(\eta^1_i, \eta^2_{ij}, 0,\dots,0)$, given the 2-mixed-coordinates of $q(x)$. $M_{sbm}$ is equivalent to the submanifold tailored by CIF, i.e. $[\zeta]_{2_t}$. The corollary follows from Proposition \ref{prop:GeometricView}.
\end{proof}
\vspace{-5mm}
\subsection{Proof of Proposition \ref{prop:sbmmlcloseform}} \label{appendix:sbmmlcloseform}
\begin{proof}
Based on Equation \ref{eq:thetaforBM}, the coordinates $[\theta_{2+}]$ for SBM is zero: $\theta_{2+}=0$.
Next, we show that the stationary distribution $p(x;\xi)$ learnt by ML has the same $[\eta_{i}^1, \eta_{ij}^2]$ with $q(x)$.

For SBM, the $\frac{\partial E(x;\xi)}{\partial \xi}$ can be easily calculated from Equation (\ref{eq:energyBM}):
$$\begin{cases} \frac{\partial E(x;\xi)}{\partial U_{x_i x_j}}=x_i x_j, & for ~ U_{x_i x_j}\in \xi; \\
\frac{\partial E(x;\xi)}{\partial b_{x_i}}=x_i, & for ~ b_{x_i}\in \xi. \end{cases}$$

Thus, based on Equation \ref{eq:learnrulestochasticgradient}, the gradients for $U_{x_i x_j},b_{x_i}\in \xi$ are as follows:
\begin{equation}\label{eq:gradientloglikelihoodsbm}
\begin{cases}
  \frac{\partial \log p(\underline{x};\xi)}{\partial U_{x_i,x_j}} = \langle x_i x_j \rangle_0 - \langle x_i x_j \rangle_\infty = \eta_{ij}^2(q(x))-\eta_{ij}^2(p(x;\xi)), & for ~ U_{x_i x_j}\in \xi; \\
  \frac{\partial \log p(\underline{x};\xi)}{\partial b_{x_i}} = \langle x_i\rangle_0 - \langle x_i\rangle_\infty = \eta_{i}^1(q(x))-\eta_{i}^1(p(x;\xi)) , & for ~ b_{x_i}\in \xi.
\end{cases}\nonumber
\end{equation}
where $\langle\cdot \rangle_0$ denotes the average using the sample data and $\langle\cdot \rangle_\infty$ denotes the average with respect to the stationary distribution $p(x;\xi)$.

Since SBM defines an $e$-flat submanifold $M_{sbm}$ of $S$ \citep{amari92igbm}, then ML converges to the unique solution that gives the best approximation $p(x;\xi)\in M_{sbm}$ of $q(x)$. When ML converges, we have $\Delta \xi\rightarrow 0$ and hence $\frac{\partial \log p(\underline{x};\xi)}{\partial \xi}\rightarrow 0$. Thus, we can see that ML converges to stationary distribution $p(x;\xi)$ that preserves coordinates $[\eta_{i}^1, \eta_{ij}^2]$ of $q(x)$. This completes the proof.
\end{proof}
\vspace{-5mm}
\subsection{Proof of Proposition \ref{prop:hqtorbm}} \label{appendix:hptorbm}
\begin{proof}
Based on the definition of divergence in Equation (\ref{eq:fisherinformationdivergence}), the following relation holds:
\begin{eqnarray}\label{eq:divergenceqp}
      D[q(x,h),p(x,h)] &=& D[q(x)q(h|x),p(x)p(h|x)] \nonumber \\
      &=& E_{q(x,h)}[log\frac{q(x)}{p(x)}+log\frac{q(h|x)}{p(h|x)}] \nonumber \\
      &=& D[q(x),p(x)] + E_{q(x)}[D[q(h|x),p(h|x)]] \nonumber
\end{eqnarray}
where $E_{q(x,h)}[\cdot]$ and $E_{q(x)}[\cdot]$ are the expectations taken over $q(x,h)$ and $q(x)$ respectively.

Therefore, the minimum divergence between $p(x,h;\xi_p)$ and $H_q$ is given as:
\begin{eqnarray}\label{eq:mindivergencehb}
      D(H_q,p(x,h;\xi_p))&=& \min_{q(x,h;\xi_q)\in H_q} D[q(x,h;\xi_q),p(x,h;\xi_p)] \nonumber \\
      &=&\min_{\xi_q}\{ D[q(x),p(x)] + E_q(x)[D[q(h|x;\xi_q),p(h|x;\xi_p)]]\} \nonumber \\
      &=& D[q(x),p(x)] + \min_{\xi_q}\{ E_{q(x)}[D[q(h|x;\xi_q),p(h|x;\xi_p)]]\} \nonumber \\
      &=& D[q(x),p(x)] \nonumber
\end{eqnarray}
In the last equality, the expected divergence between $q(h|x;\xi_q)$ and $p(h|x;\xi_p)$ vanishes if and only if $\xi_q = \xi_p$. This completes the proof.\footnote{Note that a similar path of proof is also used in Theorem 7 of \citet{amari92igbm}, which is for the fully-connected BM. Here, Here, we reformulate the proof for RBM to derive the projection $\Gamma_H(p(x,h))$.}
\end{proof}
\vspace{-5mm}
\subsection{Proof of Proposition \ref{prop:fractionalmixbijective}} \label{appendix:fractionalmixbijective}
\begin{proof}
    First, we show that $[\zeta^{xh}]$ is determined given $[\theta]$. Since there is a one-to-one correspondence between coordinates $[\theta]$ and $[p]$, $[\zeta^{xh}]$ can be directly calculated from the $p$-coordinates corresponding to $[\theta]$ based on Equation (\ref{eq:etacoordinate}) and (\ref{eq:thetacoordinate}).

    Second, $[\theta]$ is determined by knowing $[\zeta^{xh}]$. The $\{\theta^{x_ix_j}_2, \theta^{h_ih_j}_2, \theta_{2+}\}$ part of $[\theta]$ are set to be equal to those in $[\zeta^{xh}]$.
    By fixing $\{\theta^{x_ix_j}_2, \theta^{h_ih_j}_2, \theta_{2+}\}$ and setting $\{\theta^{x_i}_1, \theta^{h_j}_1, \theta^{x_ih_j}_2\}$ free, we now have an $e$-flat smooth submanifold $R$.
    Assume that there exist two different distributions $P_1$ and $P_2$ with coordinates $[\theta]_1$ and $[\theta]_2$ that have the same mixed coordinates $[\zeta^{xh}]$. Thus both $P_1$ and $P_2$ belong to $R$ and share the same value of $\{\eta_{x_i}^1,\eta_{h_j}^1, \eta_{x_ih_j}^2\}$.
    Let $Q\in S_{xh}$ be a distribution whose projection on $R$ is $P_1$. Based on the Projection Theorem in \citet{Amari93}, $P_1$ is the unique closest point on $R$ to $Q$. Considering the minimization of the divergence $D[Q,P_R]$ between $P_R \in R$ and $Q$, the gradient vector of $D[Q,P_R]$ over the free parameters $\{\theta^{x_i}_1, \theta^{h_j}_1, \theta^{x_ih_j}_2\}$ at $P_1$, that is $\{\eta_{x_i}^1, \eta_{h_j}^1, \eta_{x_ih_j}^2\}_{P_R}-\{\eta_{x_i}^1, \eta_{h_j}^1, \eta_{x_ih_j}^2\}_{Q}$, equals to zero vector.
    Then, $P_2$ also has a zero-gradient vector and hence is the projection point of $Q$, since $P_2$ has the same $\{\eta_{x_i}^1,\eta_{h_j}^1, \eta_{x_ih_j}^2\}$ with $P_1$. However, since $R$ is $e$-flat \footnote{For more information about the concept of flatness, please refer to the book \citet{Amari93}.}, the projection of $Q$ on $R$ is unique, meaning that $P_1$ and $P_2$ are the same point. Therefore, there does not exist two different distributions $P_1$ and $P_2$ that have the same mixed coordinates $[\zeta^{xh}]$. This completes the proof.
\end{proof}
\vspace{-5mm}
\subsection{Proof of Proposition \ref{prop:projectionRBMcloseform}} \label{appendix:projectionRBMcloseform}
\begin{proof}
First, we prove the uniqueness of the projection $\Gamma_B(q)$. From the $[\theta]$ of RBM in Equation (\ref{eq:thetaforRBM}), $B$ is an $e$-flat smooth submanifold of $S_{xh}$. Thus the projection is unique. Note that

Second, in order to find the $p(x,h;\xi_p)\in B$ with parameter $\xi_p$ that minimizes the divergence between $q(x,h;\xi_q)$ and $B$, the gradient descent method iteratively adjusts $\xi_p$ in the negative gradient direction that the divergence $D[q,p(\xi_p)]$ decreases fastest:
\begin{equation}\label{eq:gradientdescentMinimumDivergence}
  \triangle \xi_p=-\lambda \frac{\partial D[q,p(\xi_p)]}{\partial \xi_p} \nonumber
\end{equation}
where $D[q,p(\xi_p)]$ is treated as a function of RBM's parameters $\xi_p$ and $\lambda$ is the learning rate. As shown in \citet{highorderBM}, the gradient descent method converges to the minimum of the divergence with proper choices of $\lambda$, and hence achieves the projection point $\Gamma_B(q)$.

Last, we show that the fractional mixed coordinates $[\zeta^{xh}]_{\Gamma_B(q)}$ in Equation (\ref{eq:mixedcoordinatenewProjectionRBM}) is exactly the convergence point of the learning for RBM .
We calculate the first-order derivative of $D[q,p(\xi_p)]$, w.r.t $\xi_p$, where $p(x,h;\xi_p)=\frac{1}{Z}\exp \{-E(x,h;\xi_p)\}$ is given in Equation (\ref{eq:probBM}).

For $W_{x_i,h_j}$ in $\xi_p$ (denoted as $W_{ij}$), we have:
\begin{equation}\label{eq:divergencederivative}
  \frac{\partial D[q,p(\xi_p)]}{\partial W_{ij}} = -\sum_{x,h} \frac{q(x,h)}{p(x,h;\xi_p)} \frac{\partial p(x,h;\xi_p)}{\partial W_{ij}}
\end{equation}
where the $\frac{\partial p(x,h)}{\partial W_{ij}}$ is calculated as follows:
\begin{eqnarray}\label{eq:calcuatefirstderivative}
  \frac{\partial p(x,h;\xi_p)}{\partial W_{ij}} &=& Z^{-1}\exp\{-E(x,h)\} \{\frac{\partial (-E(x,h))}{\partial W_{ij}}-\sum_{x,h} p(x,h)\frac{\partial (-E(x,h))}{\partial W_{ij}}\}
  \nonumber \\
  &=& p(x,h)\cdot \frac{\partial (-E(x,h))}{\partial W_{ij}} - p(x,h)\sum_{x,h} p(x,h)\cdot \frac{\partial (-E(x,h))}{\partial W_{ij}} \nonumber \\
  &=&  p(x,h)\cdot x_i h_j - p(x,h) \cdot \sum_{x,h}p(x,h) x_i h_j
\end{eqnarray}

Combining Equation (\ref{eq:divergencederivative}) and (\ref{eq:calcuatefirstderivative}), we have:
\begin{equation}\label{eq:firstderivativewij}
  \frac{\partial D[q,p(\xi_p)]}{\partial W_{ij}} = -\sum_{x,h}q(x,h)x_i h_j + \sum_{x,h}p(x,h) x_i h_j = \eta_{x_ih_j}^2(p) - \eta_{x_ih_j}^2(q)
\end{equation}
where $\eta_{x_ih_j}^2(p)$ and $\eta_{x_ih_j}^2(q)$ denotes the $2^{nd}$-order $\eta$-coordinates of $p$ and $q$ respectively.

Similarly, the first-order derivatives for biases $b_{x_i}$ and $d_{h_j}$ can be proved to be:
\begin{eqnarray}\label{eq:firstderivativeb}
  \frac{\partial D[q,p(\xi_p)]}{\partial b_{x_i}} = \eta^1_{x_i}(p)-\eta^1_{x_i}(q)
\end{eqnarray}
\begin{equation}\label{eq:firstderivatived}
  \!\!\!\!\!\!\!\!\!\!\!\!\frac{\partial D[q,p(\xi_p)]}{\partial d_{h_j}} = \eta^1_{h_j}(p)-\eta^1_{h_j}(q)
\end{equation}

Summarizing Equation (\ref{eq:firstderivativewij}), (\ref{eq:firstderivativeb}) and (\ref{eq:firstderivatived}), the first-order derivatives of $\xi_{p}^{I}$, where the indexing $I=\{x_i\}~or~\{h_j\}~or~\{x_i,h_j\}~(\forall~x_i\in x,~h_j\in h)$, can be calculated in the same way, that is subtracting $p$'s and $q$'s corresponding $\eta$-coordinates $\eta_I$:
\begin{equation}\label{eq:firstderivativealinone}
  \!\!\!\!\!\!\!\!\!\!\!\!\frac{\partial D[q,p(\xi_p)]}{\partial \xi_p^I} = \eta_I(p)-\eta_I(q) \nonumber
\end{equation}

When converging, we have $\frac{\partial D[q,p(\xi_p)]}{\partial \xi_p^I}\rightarrow 0$. Hence, the gradient descent method converges to the projection point $\Gamma_B(q)$ with a stationary distribution $p(x,h;\xi_p)$ that preserves coordinates $[\eta_{x_i}^1, \eta_{h_j}^1, \eta_{x_ih_j}^2]$ of $q(x,h;\xi_q)$.
This completes the proof.
\end{proof}
\vspace{-6mm}
\subsection{Proof of Proposition \ref{prop:monotonicdivergence}} \label{appendix:monotonicdivergence}
\begin{proof}
Since $p_i\in B$ and $p_{i+1}\in B$ is the projection of $q_{i+1}$, then $D[q_{i+1},p_i] \geq D[q_{i+1},p_{i+1}]$. Similarly, $q_{i+1}\in H_q$ and $q_{i+2}\in H_q$ is the projection of $p_{i+1}$, thus $D[q_{i+1},p_{i+1}] \geq D[q_{i+2},p_{i+1}]$. This completes the proof.
\end{proof}
\vspace{-9mm}
\bibliography{cif_jmlr_13}

\end{document}